\documentclass[11pt,letterpaper]{article}

\usepackage[utf8]{inputenc}
\usepackage{amsmath}
\usepackage{amsfonts}
\usepackage{amssymb}
\usepackage{graphicx}
\usepackage{float}
\usepackage{xcolor}
\usepackage{amsthm}
\usepackage{mathtools}
\usepackage{hyperref}
\usepackage{subcaption}
\usepackage{mwe}
\usepackage{comment}

\newcommand{\E}{\mathbb{E}}

\newcommand{\bt}{\beta}
\newcommand{\R}{\mathbb{R}}
\newcommand{\1}{\mathbf{1}}

\newcommand{\PR}{\mathbb{P}}

\newcommand{\sgen}{\alpha + \beta^\top x}
\newcommand{\si}{\alpha + \beta^\top x_i}
\newcommand{\bxi}{\beta_*^\top x_i}
\newcommand{\bx}{\beta_*^\top x}

\DeclarePairedDelimiter{\norm}{\lVert}{\rVert}

\newtheorem{theorem}{Theorem}[section]
\newtheorem{corollary}{Corollary}[section]
\newtheorem{lemma}{Lemma}[section]
\newtheorem{definition}{Definition}

\newtheorem{condition}{Condition}
\newtheorem{proposition}{Proposition}[section]

\title{\LARGE\bf Linear Classifiers Under Infinite Imbalance}
\author{Paul Glasserman and Mike Li \\ Columbia Business School}
\date{June 2021; revised October 2022}

\begin{document} 

\maketitle

\begin{abstract}
We study the behavior of linear discriminant functions for binary classification in the
infinite-imbalance limit, where the sample size of one class grows without
bound while the sample size of the other remains fixed. The coefficients
of the classifier minimize an empirical loss specified through a weight function.
We show that for a broad class of weight functions, the intercept diverges but the
rest of the coefficient vector has a finite almost sure limit under infinite imbalance,
extending prior work on logistic regression. The limit
depends on the left-tail growth rate of the weight function, 
for which we distinguish two cases: subexponential and exponential.
The limiting coefficient vectors reflect robustness or conservatism properties in the sense that they
optimize against certain worst-case alternatives. In the subexponential case,
the limit is equivalent to an implicit choice of upsampling distribution for the minority class.
We apply these ideas in a credit risk setting, with particular emphasis on performance
in the high-sensitivity and high-specificity regions.
\end{abstract}

\baselineskip18pt

\section{Introduction}

Binary classification tasks often face a severe problem of imbalanced data:
observations from one class are plentiful but observations from the other
class are scarce. In detecting rare diseases, for example, one may have
access to nearly unlimited measurements from healthy patients but
only a few from sick patients; lenders often have rich data on low-risk borrowers but fewer observations of borrowers who default.

Under extreme imbalance, a simple rule that predicts that all observations
are from the majority class achieves near-perfect accuracy on training data.
Such a rule is clearly useless in detecting observations from the minority class,
which is often the main objective of the classification. 

A common approach to binary classification evaluates a scoring rule at each
observation and then applies a threshold to the scores to assign each observation
to one class or the other. In analyzing performance under imbalance we
would like to separate the choice of threshold (which can be adjusted to
correct for imbalance) from the effect of the features used in the scoring
rule.

In the case of logistic regression, this separation is provided by the intercept
and the rest of the coefficient vector. Owen \cite{owen} analyzed the
behavior of logistic regression in the infinitely imbalanced limit, where 
the size of the majority class becomes infinite while the size of the minority
class remains fixed. He showed that the intercept tends to negative infinity,
as a consequence of the growing imbalance, but the rest of the coefficient
vector approaches a finite limit.

We extend Owen's \cite{owen} result to
a wide class of linear discriminant functions exhibiting a variety of behaviors.
These classifiers assign a score to each observation using a linear
combination of features, with coefficients chosen to minimize a loss
function; the score is compared to a threshold to classify the observation.
In the framework of Eguchi and Copas \cite{eguchi}, the loss is
determined by a weight function that penalizes
high scores for one class and low scores for the other class. 
We prove infinite-imbalance limits for the coefficient vectors of a broad
family of such classifiers, with explicit expressions for the limits.
Even in the case of logistic regression, our results extend Owen's \cite{owen}
because we directly analyze the empirical loss rather than the approximation
used in Owen \cite{owen}.

We distinguish two broad categories of classifiers we call \emph{asymptotically
subexponential} and \emph{asymptotically exponential}, based on 
the left-tail growth rates of their weight functions. The first category includes
bounded weight functions, which further include logistic regression as a 
special case. We show that all classifiers in this category have the
same limit under infinite imbalance and are therefore equivalent to
logistic regression in this limiting regime.

The asymptotically exponential category includes the loss function in the 
AdaBoost method (as formulated in Freund and Schapire \cite{freund}, and Friedman, Hastie, and Tibshirani \cite{friedman}) and asymmetric extensions of this
method. For this category we show that the limiting coefficient vector
depends on the exponent in the left-tail growth rate of the weight
function. From this perspective, the limit for logistic regression can be
seen as a very special boundary case, corresponding to an exponent of zero.

The asymptotically exponential case allows a richer set of limits,
and varying the exponent in this family of methods provides useful flexibility in controlling the performance of a classifier with highly imbalanced data. By varying the exponent we can put more weight on specificity (the true negative rate) or sensitivity (the true positive rate) in the classification task.

To support this interpretation, we study
the form of the limiting coefficient vectors to understand what
the infinite-imbalance limit says about the classification rules.
We show that the limits reflect robustness properties, in the sense that they
are optimized against certain worst-case alternatives. Different types of robustness
properties can also be seen as different types of conservatism in selecting which
errors to emphasize under extreme imbalance.

This robustness or conservatism is
easiest to appreciate when the weight function is asymptotically subexponential, which includes the case of logistic regression.
We know from the Neyman-Pearson lemma that the optimal rule for classifying an 
observation as coming from one probability
distribution or another uses the likelihood ratio between the two distributions.
The limiting coefficient vector with a 
subexponential weight function is the log likelihood ratio between the distribution
of the majority class and a ``worst-case'' alternative. 
Among the set of distributions having the mean of the minority class, this alternative 
is the distribution closest to that of the majority class, with closeness measured through relative entropy or Kullback-Leibler divergence. This is the worst case because distributions
that are closer are harder to separate. Thus, we show that
the limiting coefficient vector provides the best (Neyman-Pearson) classifier for the worst alternative to the majority distribution among all distributions with the mean of the minority class. 

We also prove a version of this result when the left tail of the weight functions grows exponentially. The subexponential case implicitly emphasizes conservatism
with respect to false positives in identifying draws from the minority class. The asymptotically exponential case balances concerns about false negatives and false positives, with the relative weight determined by the choice of exponent.

Imbalance is often addressed through downsampling (discarding observations
from the majority class) or upsampling (reproducing  or creating synthetic observations
from the minority class); see, for example, the methods in Chawla et al. \cite{smote},
Drummond and Holte \cite{drummond}, and Kubat and Matwin \cite{kubat}, and
the comparison of methods in Liu, Wu, and Zhou \cite{lwz}.
The infinite imbalance limits we study can also be understood from this
perspective in the asymptotically subexponential case.
Linear discriminant rules in this case (including logistic
regression) become equivalent, in the infinite-imbalance limit, to an
implicit choice of upsampling distribution. This implicit rule upsamples
the minority class using the worst-case alternative to the majority class.

We illustrate these ideas through numerical examples and an empirical application.
As is customary, we examine classification performance through the
receiver operating characteristic (ROC) curve. We use partial area-under-the-curve (pAUC) 
measures, as introduced in McClish \cite{mcclish}, to focus attention on the high-specificity and high-sensitivity endpoints. We argue that these regions are where the choice of weights
for the discriminant function matters most.
Using exponential weight functions, we find a consistent ranking of performance according
to the size of the exponent, but the ordering flips between the high-sensitivity and
high-specificity regions. Consistent with our limiting results, the behavior of 
logistic regression becomes similar to that of a classifier from an exponential weight function
with a small exponent, as the degree of imbalance grows.

We apply these ideas in a credit risk setting using mortgage data from Freddie Mac.
We consider the problem of predicting default in the first two years of a loan, using
features available at the time the loan was made. Defaults are rare, making the data
highly imbalanced. We take the view that in an initial screening, a lender would want to
achieve a high level of sensitivity in detecting likely defaulters. We calibrate logistic
and exponential classifiers to achieve high true positive rates in training data and then
compare true positive and true negative rates in test data. The relative performance
of the exponential classifiers and logistic regression are consistent with the predictions
from our theoretical analysis.

Section~\ref{s:setup} discusses the class of linear discriminant functions we study
based on minimizing expected loss measures.
Section~\ref{s:existence} establishes the existence of unique minimizers for empirical
loss measures.
Section~\ref{s:main} presents our main theoretical results, the coefficient limits under
infinite imbalance.
Section~\ref{s:interp} discusses the interpretation of the limits.
Section~\ref{s:numerical} presents numerical results, and Section~\ref{s:credit}
develops the credit risk application. Proofs appear in appendices.

\section{Discriminant Functions}
\label{s:setup}

\subsection{Logistically Consistent Objectives}
\label{s:ec}

We consider data in which each observation takes the form of a pair $(x, y) \in \R^d \times \{0,1\}$, 
where $x$ is a vector of features or attributes, and $y$ is a binary class label. When we
introduce imbalance, 0 will label the majority class, and 1 will label the minority class.
A discriminant function $\eta(x)$ assigns a score to each feature vector $x$, 
with the intention that points from class 1 will tend to have higher scores than points from class 0, 
so that the score can be used for classifying unlabeled observations: 
an observation $x$ is predicted to be from class 1
if and only if $\eta(x)>t$, for some threshold $t$.
A linear discriminant function takes the form $\eta(x) = \alpha + \beta^{\top}x$,
for some $\alpha\in\R$ and $\beta\in\R^d$.
We are primarily interested in the
vector $\beta$; if our rule for predicting class 1 based on features $x$ is $\sgen >t$, then the effect
of $\alpha$ can be absorbed into the threshold $t$.

We select $(\alpha,\beta)$ by minimizing a loss function. To formulate the objective,
it is useful to introduce distributions $F_0$ and $F_1$ on $\R^d$, describing the distributions of features in the two classes, and marginal probabilities $\pi_0$
and $\pi_1 = 1-\pi_0$ for the two labels. Write $\mathbb{E}_i$ for
expectation with respect to $F_i$, $i=0,1$, and write $\mathbb{P}_i$ for
the corresponding probability measures. The loss function is defined
by two increasing functions $U$, $V$ on $\R$, which yield the objective
\begin{equation}
C(\alpha, \bt) = -\pi_1 \E_1[U(\alpha + \beta^\top X)] + \pi_0 \E_0[V(\alpha + \beta^\top X)].
\label{cpi}
\end{equation}
The first term on the right penalizes small scores in class 1, and the second term
penalizes large scores in class 0.
We can also write this loss function as
\begin{equation}\label{cost_general}
    C(\alpha, \bt) = \E[-Y U(\alpha + \beta^\top X) + (1-Y) V(\alpha + \beta^\top X)], 
\end{equation}
by taking the expectation with respect to the unconditional distribution of $(X,Y)$,
under which $\mathbb{P}(Y=i)=\pi_i$ and $X$ has distribution $F_i$, given
$Y=i$, $i=0,1$.

Let $\eta$ be the log likelihood ratio of $F_1$ with respect to $F_0$,
which is well-defined on the intersection of the support of $F_0$ and $F_1$,
and which is not in general linear. 
Let $p(x) = \PR(Y=1|X=x)$, and define the log-odds
\begin{equation}
\eta_o(x) = \log \frac{p(x)}{1-p(x)} = \log\frac{\pi_1}{\pi_0} + \log\frac{dF_1}{dF_0}(x) \equiv \log\frac{\pi_1}{\pi_0} + \eta(x).
\label{etadef}
\end{equation}
By the Neyman-Pearson lemma, (\ref{etadef}) provides the optimal discriminant function in the sense
that it minimizes the error probability $\PR_1(\eta(X)\le t)$ for any value of 
the error probability $\PR_0(\eta(X)> t)$, as $t$ varies.
(See, e.g., Theorem 3.2.1 of Lehmann and Romano \cite{lehmann} for a precise statement.)
The log-odds $\eta_o$ provides an equivalent classifier because it differs
from $\eta$ by a constant that can be absorbed in the threshold $t$.

The log-odds need not be linear and thus need not be achievable by minimizing (\ref{cost_general}).
Eguchi and Copas \cite{eguchi} proposed the following 
\emph{logistically consistent} restriction on $U$ and $V$:
\emph{if} the log-odds function is linear, $\eta_o(x) = \alpha_o + \beta_o^{\top}x$, then (\ref{cost_general}) should be minimized at $(\alpha_o,\beta_o)$. 
In other words, $U$ and $V$ should deliver the optimal classifier if the
optimal classifier is linear. 

Eguchi and Copas \cite{eguchi} show that this consistency condition holds if
the penalty functions $U$ and $V$ satisfy
\[
\frac{\partial V(u)}{\partial u} = e^u \frac{\partial U(u)}{\partial u};
\]
equivalently,
\begin{equation}
U(u) = C_U - \int_{u}^\infty w(s) ds,\quad  V(u) = C_V+\int_{-\infty}^u e^s w(s) ds,
\label{uvint}
\end{equation}
for some positive function $w$ and some constants $C_U$ and $C_V$.
Our analysis applies to linear discriminant functions obtained by minimizing
an empirical counterpart of
(\ref{cost_general}) with $U$ and $V$ of this form. 
The constants $C_U$ and $C_V$ have no effect in minimizing (\ref{cost_general}).
 
Before proceeding further, we briefly review some common evaluation metrics in 
binary classification that we will use later.
The \emph{sensitivity} of a classifier refers to the true positive rate, 
or the probability that a positive instance will be classified as positive.
For a discriminant function $\eta$ and threshold $t$, the sensitivity
is $\mathbb{P}_1(\eta(X)>t)$, if we interpret a positive instance to be
an observation from class 1.
The \emph{specificity} of a classifier refers to the true negative rate, 
or the probability $\mathbb{P}_0(\eta(X)\le t)$
that a negative instance will be classified as negative.  
Setting the threshold $t$ to positive infinity classifies all instances
to the negative class (class 0) and achieves $100\%$ specificity but suffers
from 0\% sensitivity. 
As one decreases the decision threshold, the classifier's sensitivity improves at 
the expense of specificity, until it reaches the other extreme point 
where all instances are classified as positive, achieving $100\%$ sensitivity
and 0\% specificity.

The intrinsic trade-off between sensitivity and specificity is often illustrated
through the \emph{receiver operating characteristic curve} (ROC curve).
The ROC curve is the set of points traced by coordinates 
$(\mathbb{P}_0(\eta(X)> t),\mathbb{P}_1(\eta(X)> t))$, $-\infty \le t \le \infty$,
connecting point $(0,0)$ when $t \to \infty$ and $(1,1)$ when $t \to -\infty$. 
The second coordinate is the sensitivity of the classifier, 
and the first is 1 minus the specificity, so higher levels of the ROC curve
indicate better performance.
We will see examples of ROC curves in Section~\ref{s:numerical}.

The Neyman-Pearson lemma implies (see Lehmann and Romano \cite{lehmann},
p.62) that the ROC curve for the log-odds classifier 
lies above the ROC curve for any other discriminant function. 
As discussed in Section 2.3 of Eguchi and Copas \cite{eguchi}, 
the loss $C(\alpha, \bt)$ can be interpreted as the weighted area between the log-odds ROC curve  
and the ROC curve for the linear classifier determined by $(\alpha,\bt)$;
the weight assigned to the gap between the curves at a score of $u$ is $w(u)$.
By minimizing $C(\alpha, \bt)$, we find the linear score $\sgen$ that is closest to the true log-odds function, 
in the sense of this weighted area.
Different weight functions balance the sensitivity-specificity trade-offs differently. 
The weight $w(u)$ at large positive $u$ emphasizes the high-specificity/low-sensitivity region of the ROC curve, and the weight $w(u)$ at large negative $u$ emphasizes the high-sensitivity/low-specificity region of the ROC curve. We return to these ideas in
Section~\ref{s:robust2} and in the application of Section~\ref{s:credit}.

\subsection{Examples of Objective Functions}
\label{s:examples}

With $w(u) = w_0(u) = 1/(1+e^u)$, and $C_U=C_V=0$, we get
\begin{equation}
U(u) = \log \frac{e^u}{1+e^u}, \quad V(u) = - \log \frac{1}{1+e^u},
\label{lr}
\end{equation}
and the loss (\ref{cost_general}) becomes
\[
C(\alpha, \bt) = - \E \left [Y \log \frac{e^{\alpha + \beta^{\top}X}}{1+e^{\alpha + \beta^{\top}X}} + (1-Y) \log \frac{1}{1+e^{\alpha + \beta^{\top}X}} \right ].
\]
Minimizing this expression (or more precisely its empirical counterpart) is equivalent to 
maximizing the likelihood function in ordinary logistic regression. 
That is, ordinary logistic regression is a special case of this family of objectives.
The discriminant functions $x\mapsto \alpha + \beta^{\top}x$ and
$x\mapsto \exp(\alpha + \beta^{\top}x)/(1+\exp(\alpha + \beta^{\top}x))$ yield equivalent
classification rules because each is a monotone transformation of the other.

Among other examples, we will also consider exponential weight functions,
\[
w(u) = \lambda (1-\lambda) e^{-\lambda u},\quad U(u) = -(1-\lambda) e^{-\lambda u},\quad V(u) = \lambda e^{(1-\lambda) u},
\]
with $\lambda \in (0,1)$, for which the loss function becomes
\begin{equation}
C(\alpha, \bt) = \E[Y (1-\lambda) e^{-\lambda (\alpha + \beta^{\top}X)} + (1-Y) \lambda e^{(1-\lambda)(\alpha + \beta^{\top}X)}].
\label{eloss}
\end{equation}
As illustrated in Figure~\ref{f:wuv}, the logistic weight $w_0$
is bounded whereas the exponential $w$ is not.
A larger $\lambda \in (0,1)$ attaches a greater penalty to large negative values of $\sgen$ when $y=1$, 
and a smaller $\lambda \in (0,1)$ attaches a greater penalty to large positive values of $\sgen$ when $y=0$. 
Informally, $\lambda$ balances a trade-off between false negative
and false positive probabilities. (Recall that we take draws from class 1 to be positive cases.)
We will see that a larger $\lambda$ is more sensitive to the distribution of the minority class.
The symmetric case $\lambda=1/2$ corresponds to the loss function behind the AdaBoost
method of Freund and Schapire \cite{freund}, as discussed in Section 4.1 of Friedman, Hastie, and Tibshirani \cite{friedman} and Section 2.4 of Eguchi and Copas \cite{eguchi}.

Our analysis also allows weight functions like
$$
w(u) = \left\{\begin{array}{ll}
1 - 2u, & u\le 0; \\
(1+u)^{-2}, & u>0,
\end{array}\right.
$$
that are in between an exponential weight function and the bounded weight
function underlying logistic regression in (\ref{lr}), in the sense that
in this example $w(u)$ grows linearly as $u\to-\infty$. We provide
precise conditions on $w$ in Section~\ref{s:main}.

\begin{figure}
\centering
\begin{subfigure}{.5\textwidth}
    \centering 
    \caption*{Logistic Case}
  \includegraphics[width=.6\linewidth]{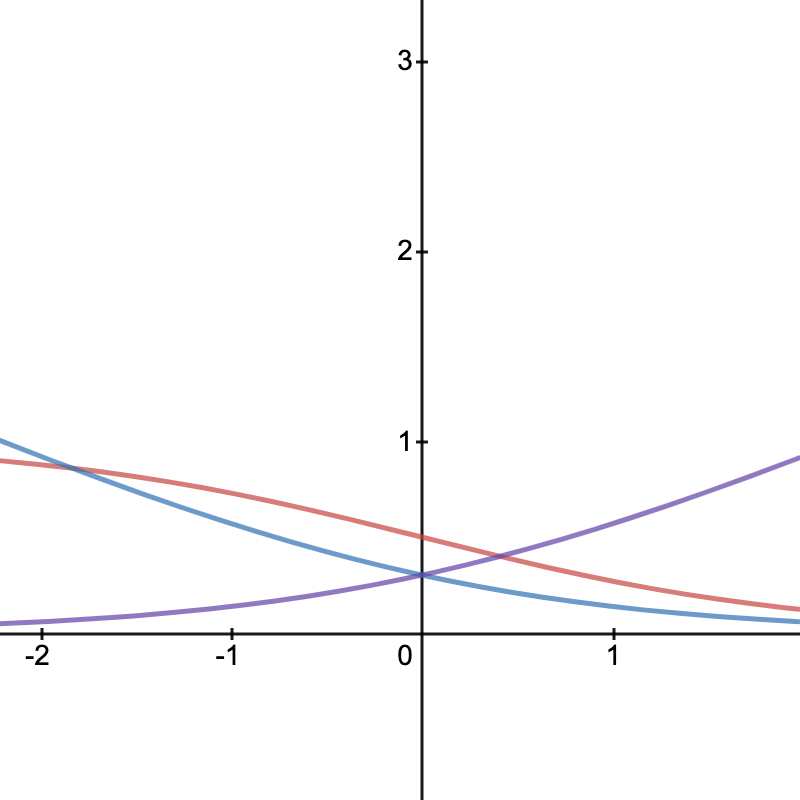}
  \caption{\textcolor{red}{$w(u) = \frac{1}{1+e^u}$}, \\
  \textcolor{blue}{$-U(u) = -\log \frac{e^u}{1+e^u}$}, \textcolor{violet}{$V(u) = -\log \frac{1}{1+e^u}$}}
\end{subfigure}%
\begin{subfigure}{.5\textwidth}
  \centering 
  \caption*{Exponential Case}
  \includegraphics[width=.6\linewidth]{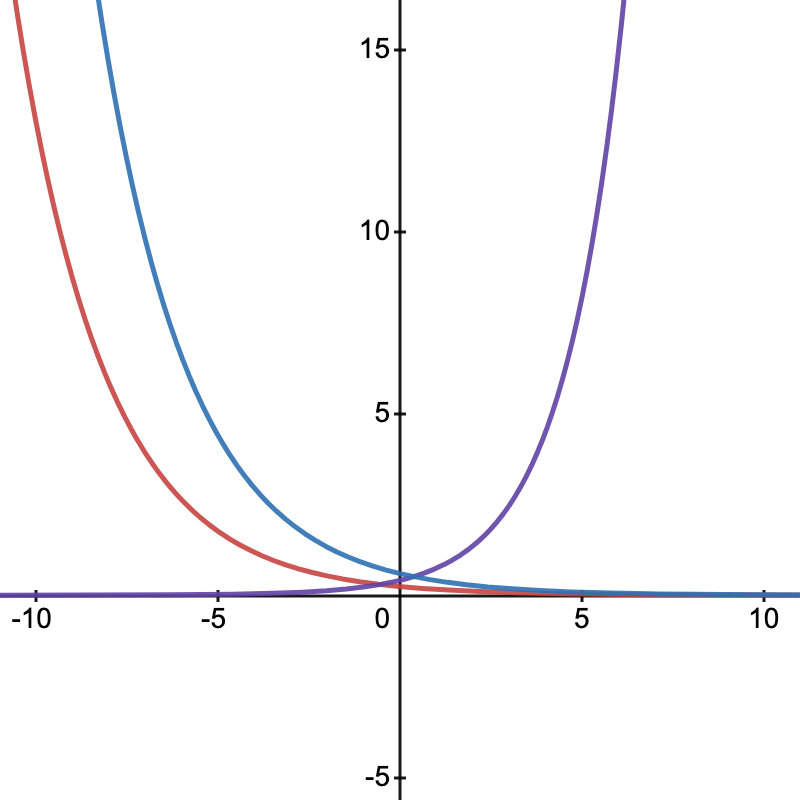}
  \caption{
  \textcolor{red}{$w(u) = \lambda (1-\lambda) e^{-\lambda u}, \lambda \in (0,1)$}, \\
  \textcolor{blue}{$-U(u) = (1-\lambda) e^{-\lambda u}$}, 
  \textcolor{violet}{$V(u) = \lambda e^{(1-\lambda) u}$}}
\end{subfigure}
\caption{Examples of weight functions $w$ and penalty functions $U$ and $V$.}
\label{f:wuv}
\end{figure}

\subsection{Empirical Loss}
In estimating $(\alpha,\beta)$, we minimize an empirical version of the loss (\ref{cost_general}).
Let $x_1, \dots, x_n$ denote $n$ observations from class 1,
and let $X_1, \dots, X_N$ denote $N$ independently and identically distributed (i.i.d.) 
samples from class 0, with underlying distribution $F_0$.
Define
\begin{equation}\label{cost}
\bar{C}_N(\alpha, \beta) = \sum_{i=1}^{n} -U(\alpha + \beta^\top x_i) + \sum_{j=1}^N V(\alpha + \beta^\top X_j).
\end{equation}
To go from (\ref{cpi}) to (\ref{cost}), replace the expectations
with sample means, replace
$\pi_1$ with $n/(n+N)$ and $\pi_0$ with $N/(n+N)$, and multiply by $n+N$.
For fixed $N$, let $(\alpha_N, \bt_N)$ minimize (\ref{cost}). We study the behavior of $(\alpha_N, \bt_N)$ as $N \to \infty$, with $n$ fixed.

\section{Existence of a Minimizer}
\label{s:existence}

In this section, we provide conditions ensuring the existence of a unique, finite minimizer
of (\ref{cost}) for all sufficiently large $N$, a.s. We first state some basic assumptions:

\begin{condition}[Basic properties]\label{assump_basic}
The weight function $w$ is strictly positive on $\R$. The penalty functions $U$ and $V$
in (\ref{uvint}) are well-defined and finite on all of $\R$. 
\end{condition}

The conditions on $U$ and $V$ imply that $w(s)\to 0$ and $e^{-s}w(-s)\to 0$,
as $s\to\infty$. For the loss to be convex,
we want $w(u)$ to be decreasing and $e^uw(u)$ to be increasing, so we assume 
\begin{condition}[Convexity] \label{assump_conv}
For all $u\in\R$, $w'(u) \le 0$ and $w(u) + w'(u) > 0$.
\end{condition}

Although weak inequalities would suffice for convexity, we make the second inequality
strict to ensure strict convexity.

For ordinary logistic regression, Silvapulle \cite{silvapulle} provides necessary and sufficient
conditions for the existence of maximum likelihood estimates. These conditions include
a requirement of overlap between the two classes.
Owen \cite{owen} proposes a stronger, but broadly applicable condition to prevent degeneracy.
The key idea is that the empirical distribution of the minority class and the true distribution $F_0$
for the majority class should overlap at least to the extent that there is probability mass (with respect to $F_0$) along every possible direction away from the empirical mean of the minority class. 
We will use this condition as well. It relies on the following
definition from Owen \cite{owen}.

\begin{definition}[Surrounding property]
The distribution $F$ on $\R^d$ surrounds the point $x^*$ if for some $\epsilon > 0$, for some $\delta > 0$ and for all $\omega \in \Omega = \{\omega \in \R^d | \omega^\top \omega = 1\}$
\begin{equation}\label{surrounding}
    \int_{(x-x^*)^\top \omega > \epsilon} d F(x) > \delta.
\end{equation}
\end{definition}

\begin{figure}
    \centering
    \includegraphics[width=6cm]{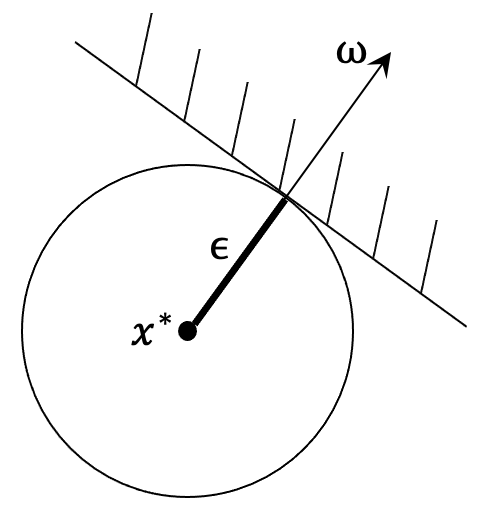}
    \caption{$F$ surrounds point $x^*$ if it assigns mass at least $\delta>0$ to
    the shaded half-space, for every direction $\omega$, for some $\epsilon>0$.}
    \label{fig:surr}
\end{figure}
Figure~\ref{fig:surr} illustrates the surrounding condition: $F$ surrounds point $x^* \in \R$ if $F$ has at least $\delta$ mass in \emph{every} half-space that is $\epsilon$ away from $x^*$. This holds, for example, if $F$ has a density that is bounded
away from zero in a neighborhood of $x^*$.

\begin{lemma}[Convex Objective]\label{cvx_obj}
Under Conditions \ref{assump_basic}--\ref{assump_conv}, $\bar{C}_N$ is 
a.s. convex.
If in addition $F_0$ surrounds at least one point in $\R^d$, 
$\bar{C}_N$ is strictly convex for all sufficiently large $N$, a.s.
\end{lemma}

It turns out that if $F_0$ surrounds the 
minority class mean $\bar{x}$, then the loss function $\bar{C}_N(\alpha, \bt)$ in (\ref{cost}) has a unique finite minimizer, as established below.

\begin{condition}[Surrounding minority mean]\label{assump_surr_mean}
$F_0$ surrounds $\bar{x}$ with some parameters $\epsilon,\delta>0$,
where $\bar{x} = (x_1+\cdots+x_n)/n$ is the mean of the minority class.
\end{condition}

\begin{lemma}[Existence]\label{existence}
Let $n \ge 1$ and $x_1,\dots,x_n \in \R^d$ be given.
If Conditions \ref{assump_basic}--\ref{assump_surr_mean} hold
then the loss function $\bar{C}_N(\alpha, \bt)$ in (\ref{cost}) has a unique finite minimizer 
$(\alpha_N, \bt_N)$ for all sufficiently large $N$, a.s.
\end{lemma}

Although the existence result Lemma~\ref{existence} only requires that $F_0$ surround the minority class mean $\bar{x}$, for our main results we will assume that $F_0$ surrounds every minority class observation $x_1, \dots, x_n$, as stated in the following condition:

\begin{condition}[Surrounding all minority observations]\label{assump_surr_all_min}
With some parameters $\epsilon^o,\delta^o>0$,
$F_0$ surrounds every minority class observation $x_1,\dots,x_n$.
\end{condition}

\section{Convergence Under Infinite Imbalance}
\label{s:main}

This section states our main result, describing the behavior of the optimal $(\alpha_N,\beta_N)$ in the limit of infinite imbalance. Our analysis considers weight functions $w(u)$ based primarily on their 
properties for large negative values of $u$. 
We are particularly interested in distinguishing weight functions that grow
exponentially or subexponentially as $u\to-\infty$, and to make this
distinction precise we introduce additional conditions.
After stating these conditions, we will show that they
are satisfied by simple and easily interpreted examples,
including all the examples in Section~\ref{s:examples}.
We use the relation $\sim$ to indicate that the ratio
of two functions converges to 1.

\begin{definition}\label{d:wti}
For a weight function $w$ we define the following conditions.
\begin{itemize}
    \item Left-tail condition: $w(u) \sim e^{-\lambda u} h(u)$, as $u \to -\infty$, for some $\lambda \in [0,1)$ and some differentiable $h: \R \mapsto \R_+$ that satisfies
    \begin{itemize}
        \item[(i)] $h(u) > 0$, for all $u\in\R$;
        \item[(ii)] $-(1-\lambda) h(u) < h'(u) \le \lambda h(u)$, for all $u\in\R$;
        \item[(iii)] $\liminf_{u \to -\infty} h'(u)/h(u) \ge 0$;
        \item[(iv)] there exists some $C > 0$ and $\xi>0$ such that for any $\epsilon > 0$, there exists some $u_0 < 0$ such that for any $u,u+s \le u_0$, $s \in \R$, we have 
        \begin{equation}\label{h-rate}
            \left |\frac{h(u+s)}{h(u)} - 1 \right | \le \epsilon \max \{C, e^{\xi|s|}\}.
        \end{equation}
    \end{itemize}
    \item Right-tail condition: If $w$ is unbounded then $\limsup_{u \to \infty} \frac{V(u)}{e^u w(u)} < \infty$. 
\end{itemize}
\end{definition}

The first condition on $h$ is natural since we require $w(u) > 0$; the second condition 
is needed for Condition~\ref{assump_conv}. The third condition implies that
$h$ grows subexponentially in the following sense:
for any $\epsilon_h > 0$, 
there exists some $u_h < 0$ such that for all $u \le u_h$, $h(u) \le C e^{-\epsilon_h u}$.  An immediate consequence of condition (iv) is that for any $s \in \R$, 
\[
\left |\frac{h(u+s)}{h(u)} - 1 \right | \to 0
\]
as $u \to -\infty$. Condition (iv) controls the speed of this convergence
as $s$ varies.

Our focus in Definition~\ref{d:wti} is on the left-tail behavior of the weight function $w$ 
because under infinite imbalance we expect $\alpha_N\to -\infty$; this limit is suggested by 
letting $\pi_1\to 0$ in (\ref{etadef}). The additional condition on the right tail 
for unbounded $w$ helps ensure that 
$\alpha_N$ indeed diverges and $\beta_N$ remains bounded.
The right-tail condition is satisfied if $e^uw(u)$ is log-convex on some interval $[u_0,\infty)$, 
and this condition is satisfied if $w(u)=Ce^{-\lambda u}$, $\lambda\in(0,1)$, or $w(u) = Cu^{-k}$, $k>0$, for large $u$; see Section~\ref{s:vb} of the appendix for a more general result.

The left-tail condition in Definition~\ref{d:wti} is satisfied by the following
examples.

\begin{lemma}\label{l:w_examples}
    The following weight functions satisfy the left-tail condition in Definition~\ref{d:wti}:
    \begin{itemize}
        \item $w(u) \sim C$ as $u \to -\infty$, for some $C > 0$;
        \item $w(u) \sim C |u|^k$ as $u \to -\infty$, for some $C > 0$, $k \ge 0$;
        \item $w(u) \sim C e^{-\lambda u} |u|^k$ as $u \to -\infty$, for some $C > 0$,  $\lambda \in (0,1)$, and $k \ge 0$.
    \end{itemize}
\end{lemma} 
The first two examples in Lemma~\ref{l:w_examples} correspond to taking 
$\lambda = 0$ with $h(u) = C$ and $h(u) = |u|^k$, respectively, in Definition~\ref{d:wti}; 
the last example corresponds to taking $\lambda \in (0,1)$ and $h(u) = |u|^k$. 
For logistic regression (\ref{lr}), the weight function is bounded, so the
first case in Lemma~\ref{l:w_examples} applies, and the right-tail condition
in Definition~\ref{d:wti} is not needed.

In our main result, Theorem~\ref{bdd-thm}, the exponent $\lambda$
in Definition~\ref{d:wti}
determines the behavior of linear classifiers under infinite imbalance.
It will be useful to distinguish the following two cases:
\begin{itemize}
    \item {\bf asymptotically subexponential case:} weight functions with $\lambda = 0$. 
    \item {\bf asymptotically exponential case:} weight functions with $\lambda \in (0,1)$.
\end{itemize}
The first two examples in Lemma~\ref{l:w_examples} are in the 
subexponential category, and the last example is in the exponential category. 
For these categories to be meaningful, we need to ensure that $\lambda$
is well-defined, which we do through the following result.

\begin{proposition}\label{prop:id}
Suppose $w$ satisfies Definition~\ref{d:wti} with $w(u) \sim e^{-\lambda u} h(u)$ as $u \to -\infty$. 
Suppose $w$ also satisfies Definition~\ref{d:wti} with
$w(u) \sim e^{-\tilde{\lambda} u} \tilde{h}(u)$ as $u \to -\infty$.
Then $\lambda = \tilde{\lambda}$ and $h(u) \sim \tilde{h}(u)$ as $u \to -\infty$.
\end{proposition}

For our main result, we require the following tail condition on 
the majority class distribution $F_0$:
\begin{condition}[Tail condition]\label{assump_tail} 
For some $r>\max\{1,1-\lambda+\xi\}/\gamma$,
\begin{equation}\label{tail}
    \int e^{r\|x\|} \, d F_0(x) < \infty,
\end{equation}
where
$\gamma = (1-\lambda) \epsilon\delta$, with $\epsilon,\delta>0$
the surrounding parameters in Condition \ref{assump_surr_mean},
and where $\lambda$ and $\xi$ are the parameters from Definition~\ref{d:wti}.
\end{condition}

This condition is satisfied by distributions $F_0$ with bounded support, with Gaussian tails,
or with tails that decay exponentially at a rate faster than $r$.
A larger $\xi$ imposes a weaker condition in (\ref{h-rate}) but a stronger
condition in Condition~\ref{assump_tail}.
For the second and third cases in Lemma~\ref{l:w_examples} we can take
$\xi>0$ arbitrarily small, and for bounded $w$ we can take $\xi=0$.
We can now show that $\beta_N$ converges a.s. under
infinite imbalance and we can identify its limit.

\begin{theorem}\label{bdd-thm}
Suppose Conditions \ref{assump_basic}--\ref{assump_tail} hold
and $w$ satisfies Definition~\ref{d:wti} with $w(u) \sim e^{-\lambda u}h(u)$ as $u \to -\infty$.
Then the minimizer $(\alpha_N, \bt_N)$ of $\bar{C}_N$ in (\ref{cost}) satisfies $\alpha_N \to -\infty$ 
and $\bt_N \to \bt_*$, a.s., where $\bt_*$ is the unique solution to 
\begin{equation}\label{bstar:exp}
    \begin{aligned}
    \frac{\int x e^{(1-\lambda) \bx} dF_0(x)}{\int e^{(1-\lambda) \bx} dF_0(x)}
=    \frac{\sum_{i=1}^n x_i e^{-\lambda \bxi}}{\sum_{i=1}^n e^{-\lambda \bxi}}.
    \end{aligned}
\end{equation}
In particular, when $\lambda = 0$, $\beta_*$ is the unique solution to 
\begin{equation}\label{bstar:sub}
    \begin{aligned}
    \frac{\int x e^{\bx} dF_0(x)}{\int e^{\bx} dF_0(x)} = \bar{x},
    \end{aligned}
\end{equation}
where $\bar{x} = n^{-1}\sum_{i=1}^n x_i$ is the minority class mean.
\end{theorem}

The second part of
Theorem~\ref{bdd-thm} extends Owen's \cite{owen} result by showing that a broad class of linear discriminant functions are asymptotically equivalent to logistic regression under infinite imbalance:
all subexponential ($\lambda=0$) weight functions share the same limiting coefficient vector $\beta_*$  in (\ref{bstar:sub}).
In contrast, asymptotically exponential weight functions allow a wider range of limits
(\ref{bstar:exp}), dependent on the exponent $\lambda$. 
The limiting $\beta_N$ in (\ref{bstar:sub})
depends on the minority class observations only through their mean
$\bar{x}$.
With $\lambda\in(0,1)$, (\ref{bstar:exp})
shows that the limiting $\beta_N$ will depend on the full empirical
distribution of $x_1,\dots,x_n$.

We can interpret the left side of (\ref{bstar:sub}) as the mean of the
distribution obtained after applying an exponential tilt or reweighting to
$F_0$. Equation (\ref{bstar:sub}) then says that $\beta_*$ tilts $F_0$
to the mean of the minority class. Equation (\ref{bstar:exp})
tilts both $F_0$ and the empirical distribution of $x_1,\dots,x_n$
to a common mean. We build on this interpretation in the next section.

Even in the case of logistic regression, Theorem~\ref{bdd-thm} extends
Owen's \cite{owen} result because our $(\alpha_N,\beta_N)$ minimize
the empirical loss function (\ref{cost}) rather than an approximate loss function
that replaces the empirical distribution over $X_1,\dots,X_N$ with
the population distribution $F_0$.

Although our focus is on the behavior of $\beta_N$, as a byproduct of
our analysis we show that $\alpha_N$ diverges at least as fast as
$-\log N$. See Corollary~\ref{c:alog} in the appendix.

To see that not all choices of $w$ lead to similar limits, we briefly consider a case suggested in
Section 2.4 of Eguchi and Copas \cite{eguchi}, in which the weight function
is replaced by a measure. 
Taking $w$ to be a delta function
with unit mass at a point $u_0$, they arrive at the objective 
\begin{equation}
C(\alpha,\beta) = \pi_1\PR_1(\alpha+\beta^{\top}X\le u_0) + \pi_0e^{u_0}\PR_0(\alpha+\beta^{\top}X>u_0).
\label{cdelta}
\end{equation}
Through appropriate choice of $u_0$, $C(\alpha,\beta)$ can be interpreted as 
balancing two types of misclassification costs. 
However, we show in Appendix~\ref{s:delta} that the resulting
discriminant function degenerates under imbalance, in the sense that for all sufficiently
large $N$, $\beta_N=0$ a.s. and $\alpha_N$ can be any value less than or equal to $u_0$.

Li, Belloti, and Adams \cite{yazheLi} consider regularized logistic regression with an
$L_1$ or $L_2$ penalty on $\beta$. They show that the optimal $\beta_N$ converges
to zero under infinite imbalance. In this sense, the regularized discriminant
function degenerates under imbalance.

\section{Robustness Interpretation of \texorpdfstring{$\bt_*$}{Lg}}
\label{s:interp}

We now turn to the interpretation of the limits $\beta_*$ defined by (\ref{bstar:sub}) and (\ref{bstar:exp}). 
We will show that these limits reflect robustness properties,
in the sense that the coefficients are optimized against certain worst-case errors or combinations of errors. 
These robustness properties reflect implicit choices
of conservatism towards different types of errors.

\subsection{Asymptotically Subexponential Weight Functions}
\label{s:robust1}

The robustness interpretation is easiest to formulate in the limit (\ref{bstar:sub}) 
for the subexponential case, which includes logistic regression. 
For $\beta\in\R^d$, define the cumulant generating function of $F_0$ by setting
\begin{equation}
\psi(\bt) = \log \int e^{\bt^\top x} d F_0(x),
\label{cgf}
\end{equation}
and let $B_{\psi} = \{\beta: \psi(\beta)<\infty\}$. 
Define the exponential family of distributions $
F_{\beta}$, $\beta\in B_{\psi}$, on $\R^d$
by setting
$$
dF_{\beta}(x) = e^{\beta^{\top}x - \psi(\beta)}\, dF_0(x).
$$
Part of the content of Theorem~\ref{bdd-thm} is that $\beta_*\in B_{\psi}$;
the normalizing factor $e^{\psi(\beta_*)}$ is the denominator in (\ref{bstar:sub}).
Write $F_*$ for the special case $F_{\beta_*}$. Equation (\ref{bstar:sub}) tells us that
\begin{equation}
\int x \, dF_*(x) = \int x e^{\beta_*^{\top}x - \psi(\beta_*)}\, dF_0(x) = \bar{x},
\label{fstarm}
\end{equation}
so the mean of $F_*$ is $\bar{x}$.
We can also write (\ref{fstarm}) as $\nabla\psi(\beta_*) = \bar{x}$.
The robustness interpretation comes from identifying $F_*$ as a worst-case alternative and identifying $\beta_*$
as the optimal classifier for this worst-case alternative. 

For distributions $G$ and $F$ on $\R^d$, define the relative entropy (or Kullback-Leibler
divergence),
\begin{equation}
D(G \|F) = \int \log\frac{dG}{dF}\, dG,
\label{kldef}
\end{equation}
with $D(G \|F)=\infty$ if the support of $G$ is not contained within the support of $F$.
Relative entropy is always non-negative, and it is zero if and only if $G$ and $F$ coincide.
It is not symmetric, but $D(G \|F)$ can be interpreted as a measure of the ``distance'' of $G$
from $F$. If $D(G \|F)$ is small then $G$ is close to $F$, making the problem of discriminating
between the two distributions difficult.
Kullback and Leibler \cite{kl} interpret $D(G \|F)$ as the mean information for discriminating
between $G$ and $F$ per observation from $G$.

Let $\mathcal{M}_{\bar{x}}$ be the set of probability distributions on $\R^d$ with mean $\bar{x}$.
We know from (\ref{fstarm}) that $F_*\in \mathcal{M}_{\bar{x}}$. 
In fact, of all elements of $\mathcal{M}_{\bar{x}}$, the one ``closest" to $F_0$ 
with respect to relative entropy is $F_*$, as the following result shows. 
This result is a special case of Corollary 3.1 of Csiszar \cite{csiszar}.
\begin{lemma}\label{l:proj}
The problem 
\begin{equation}\label{min_p1}
    \min_G D(G \|F_0) \text{ subject to } \int x\ dG(x) = \bar{x},
\end{equation}
where $G$ is a probability distribution on $\R^d$, is solved by $G = F_*$. 
\end{lemma}

Lemma~\ref{l:proj} leads to a robustness property of $\beta_*$.
Let $G$ be any distribution on $\R^d$ with the same support as $F_0$. 
As discussed at the end of Section~\ref{s:ec}, the
Neyman-Pearson lemma implies that the ROC curve defined by the log likelihood 
ratio $\eta(x) = \log dG(x)/dF_0(x)$ 
lies above the ROC curve for any other discriminant function. 

In the case of the distribution $F_*$ defined by $\beta_*$, the log likelihood ratio is given by
\begin{equation}
\eta_*(x) = \log\frac{dF_*}{dF_0}(x) = \beta_*^{\top}x - \psi(\beta_*).
\label{estar}
\end{equation}
The linear discriminant function $\beta_*^{\top}x$ coincides with the optimal classifier
$\eta_*$. The constant $\psi(\beta_*)$ shifts the threshold $t$, but the two functions
trace the same ROC curve. That is, the limit in (\ref{bstar:sub}) picks the optimal
classifier for discriminating between $F_0$ and $F_*$.

Combining this lemma with (\ref{estar}), we arrive at the following conclusion:
The limiting coefficient $\beta_*$ in (\ref{bstar:sub}) provides the optimal classifier
for the worst-case alternative to the majority class distribution $F_0$, among all distributions
with the same mean $\bar{x}$ as the observations from the minority class.
The distribution $F_*$ presents the worst case because it is hardest to distinguish from $F_0$, in the sense of Lemma~\ref{l:proj}.
This robustness property does not necessarily translate to better performance;
optimizing against a worst case can be overly conservative if the true distribution
of the minority class distribution is very different from $F_*$.

The definition in (\ref{kldef}) suggests an interpretation of the conservatism 
implicit in the focus on $\beta_*$ and $F_*$.
In making $D(G \|F_0)$ smaller, we are, roughly speaking, making the optimal discriminant
function (the log likelihood ratio $\log (dG/dF_0)$) smaller at observations that have 
higher probability under $G$. In other words, in focusing on $F_*$ we are focusing on 
a distribution whose optimal classifier (relative to $F_0$) will have low sensitivity. 
This perspective suggests that the limiting coefficient
$\beta_*$ in (\ref{bstar:sub}) --- the optimal classifier for $F_*$ --- 
is implicitly conservative in classifying to the minority class and
should perform better at low-sensitivity (high-specificity) thresholds
than at high-sensitivity thresholds. 
This interpretation will be supported by our discussion of exponential weight functions
in Section~\ref{s:robust2} and the numerical results of Sections~\ref{s:numerical}
and~\ref{s:credit}, because the subexponential case considers only
$\lambda=0$ whereas the exponential case considers all $\lambda\in(0,1)$.

\subsection{Infinite Upsampling}
\label{s:up}

The problem of imbalance is sometimes dealt with in practice through upsampling --- creating
artificial data from the minority class. We now show that for the class of linear discriminant
functions defined by minimizing (\ref{cpi}) with asymptotically subexponential $w$
(including logistic regression), the estimate of $\beta$ obtained
in the infinite imbalance limit coincides with the estimate obtained from a specific choice 
of upsampling distribution, namely $F_*$. 

Suppose, then, that we have infinitely upsampled the minority class so that 
a fraction $\pi_1$ of our data is drawn from $F_*$.  Suppose we also have infinitely
many observations, a fraction $\pi_0 = 1-\pi_1$ of the total, from $F_0$.
The loss function then takes the form (\ref{cpi}). 
Differentiating (\ref{cpi}) with respect to $\beta$
and recalling that $U'(s)=w(s)$, $V'(s) = e^sw(s)$, we get the first-order condition
$$
-\pi_1\mathbb{E}_1[w(\alpha+\beta^{\top}X)X] + \pi_0 \mathbb{E}_0 [e^{\alpha+\beta^{\top}X}w(\alpha+\beta^{\top}X)X] = 0,
$$
where $\E_1$ is now expectation with respect to $F_*$. Using the likelihood ratio
$dF_*/dF_0$, as in (\ref{estar}), we can write the first term as an expectation with
respect to $F_0$ to get
$$
-\pi_1\E_0[e^{\beta_*^{\top}X-\psi(\beta_*)}w(\alpha+\beta^{\top}X)X] + \pi_0 \mathbb{E}_0 [e^{\alpha+\beta^{\top}X}w(\alpha+\beta^{\top}X)X] = 0.
$$
This equation is solved by taking
$$
\beta=\beta_*, \quad e^{\alpha} = (\pi_1/\pi_0)e^{-\psi(\beta_*)}.
$$
The limiting coefficient vector $\beta_*$ in (\ref{bstar:sub}) is thus
precisely what one would obtain through infinite upsampling of the minority class 
using $F_*$ when we have infinitely many observations from the majority class.
The relative degree of upsampling, as reflected in $\pi_1/\pi_0$, affects $\alpha$ but not $\beta$.

The choice of $F_*$ as an upsampling distribution is not arbitrary. We know from Section~\ref{s:robust1} 
that $F_*$ is the distribution ``closest'' to $F_0$ among all distributions with the mean $\bar{x}$ 
in the data from the minority class. We have argued that the optimal classifier $\beta_*$ corresponding to 
$F_*$ is implicitly conservative in classifying into the minority class. In this sense, 
upsampling according to $F_*$ is conservative if one is 
particularly concerned about classifying with high specificity.

\subsection{The Gaussian Case}

The Gaussian case provides some convenient simplifications and helps illustrate the transition from 
the subexponential limit in (\ref{bstar:sub}) to the exponential limit in 
(\ref{bstar:exp}). We begin by reviewing some properties of the Gaussian setting.

Suppose that $F_0$ is the multivariate normal distribution 
$N(\mu_0,\Sigma_0)$ on $\mathbb{R}^d$, with mean $\mu_0$ and full-rank covariance matrix
$\Sigma_0$. The cumulant generating function (\ref{cgf}) becomes
$$
\psi(\beta) = \log\int e^{\beta^{\top}x}\, dF_0(x)
= \beta^{\top}\mu_0 + \frac{1}{2}\beta^{\top}\Sigma_0\beta,
$$
for all $\beta\in\mathbb{R}^d$; this is a special case of Brown \cite{brown}, Example 1.14. 
The tilted distribution defined by
\begin{equation}
e^{\beta^{\top}x - \psi(\beta)}dF_0(x)
\label{efam}
\end{equation}
is multivariate normal with mean $\mu_0 + \Sigma_0\beta$. In other words,
the ratio of the $N(\mu_0+\Sigma_0\beta,\Sigma_0)$ density to the
$N(\mu_0,\Sigma_0)$ density is the exponential factor in (\ref{efam}).
Every point in $\mathbb{R}^d$ can be expressed as $\mu_0+\Sigma_0\beta$,
for some $\beta\in\mathbb{R}^d$, so as $\beta$ ranges over $\mathbb{R}^d$,
(\ref{efam}) ranges over all multivariate normal distributions with covariance matrix $\Sigma_0$. 
That is, the multivariate normal distributions with covariance matrix $\Sigma_0$ and arbitrary mean
form an exponential family with parameter $\beta$.

Setting aside the issue of imbalance for a moment, let $F_1$ be the multivariate
normal distribution $N(\mu_1,\Sigma_0)$, for some mean vector $\mu_1$.
Let $\beta_*$ be as in (\ref{bstar:sub}), but with $\bar{x}$ replaced by
$\mu_1$. Then $\beta_*$ is the parameter that makes the mean
of (\ref{efam}) equal to $\mu_1$, or 
$\mu_1 = \mu_0 + \Sigma_0\beta^*$, and therefore
\begin{equation}
\beta_* = \Sigma_0^{-1}(\mu_1-\mu_0).
\label{blda}
\end{equation}

We recognize (\ref{blda}) as the coefficient vector in
classical linear discriminant analysis. (See, e.g., Theorem 6.4.1 of
Anderson \cite{anderson}.)
Linear discriminant analysis classifies an observation $x$ to class 1 or 0
depending on whether
$\beta_*^{\top}x$
is larger or smaller than some threshold. We ignore the intercept in the
discriminant function because it can be
absorbed into the choice of threshold.
Thus, the coefficient vector defined by (\ref{bstar:sub}) when $F_0$
is multivariate normal coincides with the
coefficient vector in linear discriminant analysis that applies when $F_0$ and $F_1$
are multivariate normal with a common covariance matrix.
That is, when $F_0$ is multivariate normal, 
the infinite imbalance limit (\ref{bstar:sub}) chooses $\beta_*$
\textit{as if} $F_1$ were multivariate normal with the same covariance
matrix as $F_0$.
This is a special case of the upsampling interpretation in Section~\ref{s:up}
because the $\beta_*$ we get from (\ref{bstar:sub}) is the same coefficient
vector we would get if we upsampled the minority class to the distribution
$N(\bar{x},\Sigma_0)$ and then applied logistic regression.

In the left panel of Figure~\ref{f:norm}, the ellipses show probability contours for
two bivariate normal distributions $N(\mu_i,\Sigma_i)$, $i=0,1$,
$$
\mu_0 = \left(\begin{array}{c}
2 \\
2
\end{array}\right),
\quad
\Sigma_0 = \left(\begin{array}{cc}
1.4^2 & 1.4\cdot 2.2\cdot 0.6 \\
1.4\cdot 2.2\cdot 0.6 & 2.2^2
\end{array}\right),
\quad
\mu_1 = \left(\begin{array}{c}
6 \\
8
\end{array}\right),
\quad
\Sigma_1 = \left(\begin{array}{cc}
2 & 0 \\
0 & 1
\end{array}\right),
$$
Class 0 is indicated in blue,
and class 1 in red.
The blue line illustrates a set of points $x$ with a constant value of
$\beta_*^{\top}x$, where the vector $\beta_* = \Sigma_0^{-1}(\mu_1-\mu_0)$
is computed from the perspective of $F_0$, using the covariance matrix 
$\Sigma_0$.
The intercept of the line is arbitrary;
the classification rule defined by $\beta_*$ should be thought of as the set of lines
parallel to the blue line, with different intercepts corresponding to different
classification thresholds. 
The only property of $F_1$ used in calculating $\beta_*$ is the
mean $\mu_1$, just as in (\ref{bstar:sub}).
With different covariance matrices for the two classes, linear discriminant analysis is
still useful but it loses optimality properties (Anderson \cite{anderson}, Section 6.10.2); the choice of $\beta_*$ in the infinitely
imbalanced limit (\ref{bstar:sub}) fails to capture the difference in covariance
matrices and is akin to assuming that $F_1$
has the same covariance as $F_0$.
The red line similarly illustrates a classification boundary, but now using
$\beta_* = \Sigma_1^{-1}(\mu_0-\mu_1)$ calculated with
the roles of $F_0$ and $F_1$ reversed.

\begin{figure}
\centerline{\includegraphics[width=3in]{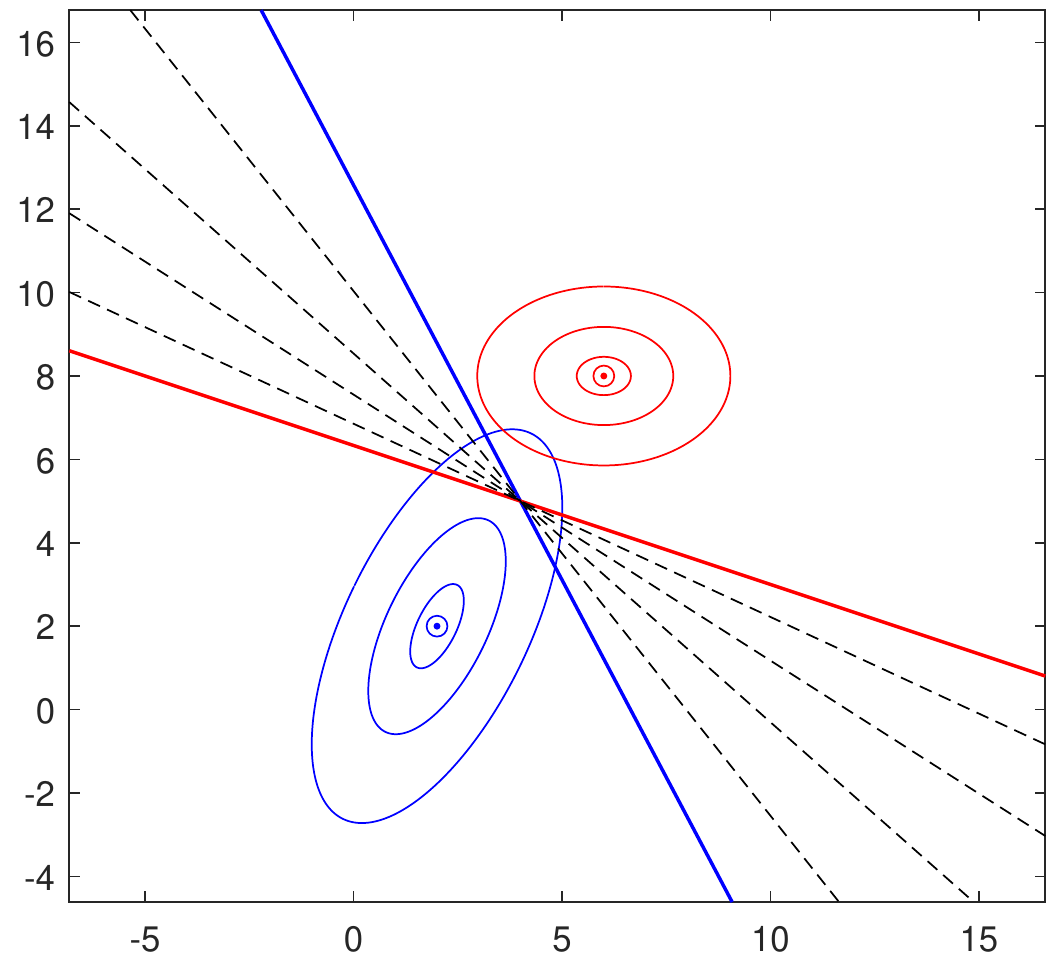}\includegraphics[width=3in]{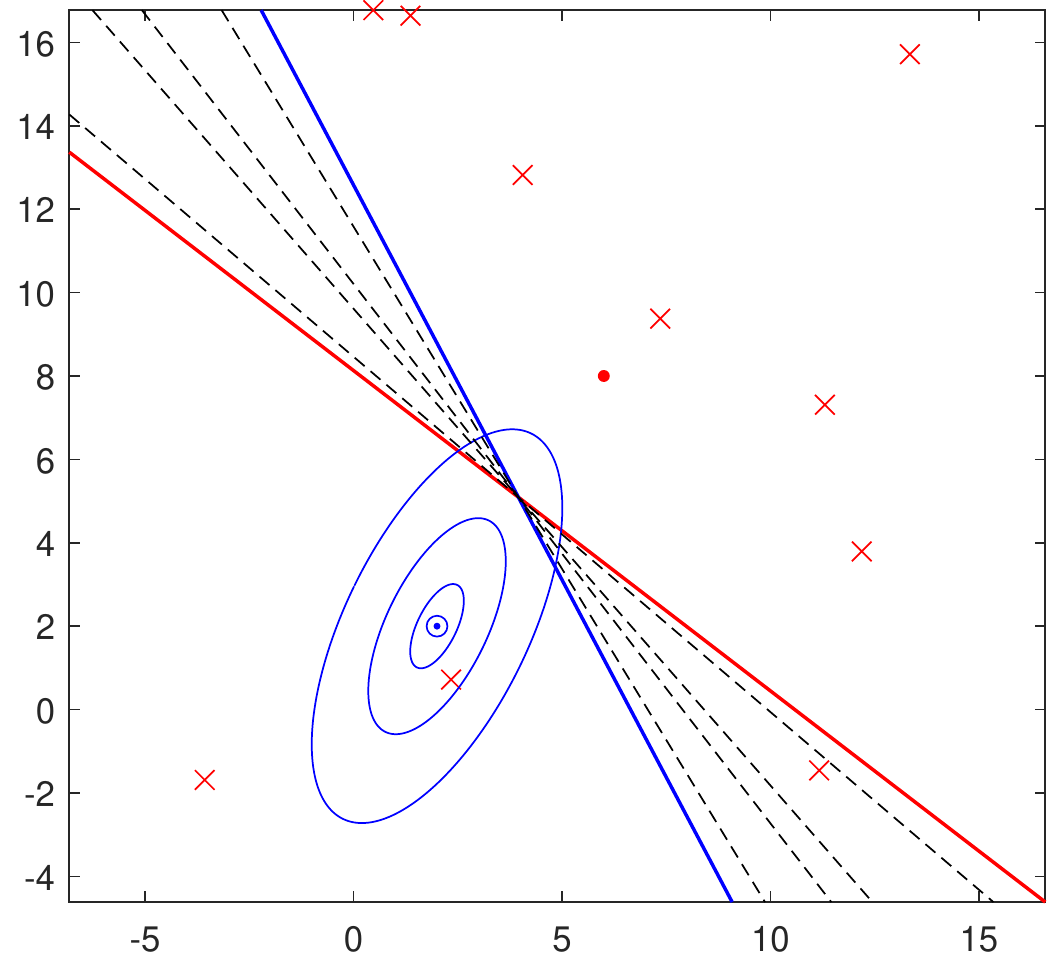}}
\caption{Classification boundaries with unequal covariance matrices.
Dashed lines show the effect of varying $\lambda$.}
\label{f:norm}
\end{figure}

Now consider the analog of (\ref{bstar:exp}) in which
the distributions on the two sides ($F_0$ on the left, the empirical distribution on $x_1,\dots,x_n$ on the right) are replaced with the normal distributions $F_0$ and $F_1$. 
In this case (\ref{bstar:exp}) reduces to 
$$
\mu_0 + (1-\lambda)\Sigma_0\beta_* = \mu_1 - \lambda\Sigma_1\beta_*;
$$
i.e.,
\begin{equation}
\beta_* = (\lambda\Sigma_1 + (1-\lambda)\Sigma_0)^{-1}(\mu_1-\mu_0).
\label{bminmax}
\end{equation}
The exponent $\lambda$ in the asymptotically exponential weight function $w$
balances the relative influence of the two distributions in setting the slope of the classifier.
The dashed lines in the left panel of Figure~\ref{f:norm} show the effect of varying 
$\lambda$ between 0 and 1.

In classical linear discriminant analysis, Anderson \cite{anderson},
p.247, shows that a coefficient vector of the form (\ref{bminmax})
solves a minimax problem of balancing the two types of misclassification
errors that can occur in binary classification.
Based on the discussion in Section~\ref{s:robust1}, we interpret
the extremes $\lambda=0$ and $\lambda=1$ as two forms
of conservatism, from the perspectives $F_0$ and $F_1$, respectively.

Now replace the normal distribution for $F_1$ with the empirical distribution on $x_1,\dots,x_n$,
while keeping $F_0$ normal. In this case (\ref{bstar:exp}) becomes
$$
\mu_0 + (1-\lambda)\Sigma_0\beta_* = 
\frac{\sum_{i=1}^n x_i e^{-\lambda \bxi}}{\sum_{i=1}^n e^{-\lambda \bxi}},
$$
which can be solved numerically. The crosses in the right panel of Figure~\ref{f:norm}
represent $x_1,\dots,x_n$, $n=10$. 
Their mean is at $(6,8)$, just as it is for the red normal distribution
in the left panel. The slopes of the solid blue lines are therefore the same in the two panels:
the discriminant computed from the blue distribution ($\lambda=0$) depends on the red distribution only through its mean. 
At $\lambda=0$, the infinite imbalance limit chooses
$\beta_*$ as if $F_1$ were the multivariate normal distribution
$N(\bar{x},\Sigma_0)$.
The solid red line ($\lambda=1$) similarly depends on $F_0$ only
through $\mu_0$. 

The dashed black lines correspond to intermediate values of $\lambda$. 
These make use of the distributions of both classes, not just their means, indicating
a potential advantage of the exponential objective over logistic regression and other
asymptotically subexponential cases.
Although it is not evident from the figure, the slopes of the lines need not change monotonically with
$\lambda$, nor is the slope at an intermediate $\lambda$ necessarily between the
slopes of the blue and red lines.

The explicit expressions in (\ref{blda}) and (\ref{bminmax}) are potentially
useful in non-Gaussian settings as starting values for numerical calculation
of optimal coefficient vectors using estimated means and covariance matrices.
Related suggestions are made in Owen \cite{owen} and Deo and Juneja \cite{deo}.

\subsection{Asymptotically Exponential Weight Functions}
\label{s:robust2}

We have seen that in the limit (\ref{bstar:sub}) for asymptotically subexponential weight functions, 
the linear classifier under infinite imbalance is 
implicitly optimized for low sensitivity and high specificity.
In this section, we extend this interpretation to the asymptotically exponential case 
(\ref{bstar:exp}) with $\lambda \in (0,1)$,
and we argue that the exponent $\lambda$ balances 
the classifier's emphasis along the sensitivity-specificity trade-off.

For distributions $F_0$, $F_1$ on $\R^d$, define the cumulant 
generating functions $\psi_i$, $i=0,1$, as in (\ref{cgf}),
with domains $B_i$, $i=0,1$, and define the exponential families of distributions
$$
dF_{i,\beta}(x) = e^{\beta^{\top}x - \psi_i(\beta)}\, dF_i(x), \quad \beta\in B_i.
$$
\begin{proposition}\label{p:dmix}
For $\lambda\in(0,1)$, suppose there is a $\beta_* \in B_0\cap B_1$ for which
\begin{equation}
\nabla\psi_0((1-\lambda)\beta_*) = \nabla\psi_1(-\lambda\beta_*).
\label{psistar}
\end{equation}
Then the problem
\begin{equation}
\min_{G_0,G_1} \lambda D(G_0 \|F_0) + (1-\lambda)D(G_1 \|F_1)
\mbox{ subject to }
\int x\, dG_0(x) = \int x\, dG_1(x),
\label{dmix}
\end{equation}
where $G_0$ and $G_1$ are distributions on $\R^d$,
is solved by $G_0=F_{0,(1-\lambda)\beta_*}$ and $G_1=F_{1,-\lambda\beta_*}$.
\end{proposition}

Equation (\ref{psistar}) generalizes (\ref{bstar:exp}); it reduces to (\ref{bstar:exp}) when
$F_1$ is the empirical distribution on $x_1,\dots,x_n$ (and the existence
of $\beta_*$ solving (\ref{bstar:exp}) is proved as part of Theorem~\ref{bdd-thm}).
The limiting coefficient vector
in (\ref{bstar:exp}) can therefore be interpreted as the result of minimizing the
objective in (\ref{dmix}), with this substitution for $F_1$. 
Proposition~\ref{p:dmix} then generalizes Lemma~\ref{l:proj}.
Indeed, as $\lambda$ approaches $0$, $D(G_1\|F_1)$ dominates the objective function, 
so the minimizer $G_1$ approaches $F_1$. In particular, the mean of $G_1$ approaches the mean of $F_1$. 
Then in solving for $G_0$, we are solving $\min_{G_0} D(G_0 \| F_0)$ 
subject to $\int x dG_0(x) = \int x d F_1(s)$, reducing to the minimization problem 
to Lemma~\ref{l:proj}.
Whereas the limit of the asymptotically subexponential case implicitly focuses on 
the worst-case distribution from the perspective of $F_0$, the objective
in (\ref{dmix}) balances the worst case as seen from both $F_0$ and $F_1$. 
In doing so, it balances the focus on the high-specificity and high-sensitivity
regions.

To make this balance more explicit, 
recall from the discussion at the end of Section~\ref{s:ec} 
that the weight function $w(u)$ can be interpreted as a penalty on the 
difference between the ROC curves for the optimal discriminant function 
and an approximating linear discriminant function. With $w(u)$ proportional to $e^{-\lambda u}$, 
we thus expect a better approximation at large negative $u$ (the high-sensitivity region of the ROC curve) when $\lambda$ is close to 1, and a better approximation 
at large positive $u$ (the high-specificity region of the ROC curve) when $\lambda$ is close to 0:
\begin{eqnarray*}
\lambda \approx 0 &\Rightarrow & \mbox{emphasizes high-specificity region;} \\
\lambda \approx 1 &\Rightarrow & \mbox{emphasizes high-sensitivity region.}
\end{eqnarray*}
This pattern is what we find in the experiments of Sections~\ref{s:numerical} and~\ref{s:credit}. 
In Appendix~\ref{s:robust_nonlinear}, we provide further discussion on the connection 
between (\ref{dmix}) and the original classification problem by generalizing 
the minimization of (\ref{eloss}) over possibly nonlinear discriminant functions.

The symmetric case $\lambda = 1/2$ admits a further interpretation.
For any discriminant function, the area under the curve measure AUC, 
discussed further in the next section, equals the probability that a draw $X_1$ from $F_1$
scores higher than an independent draw $X_0$ from $F_0$. Thus,
for a linear discriminant function $x\mapsto\beta^{\top}x$,
Markov's inequality yields
$$
\mbox{AUC} = 1-\PR(\beta^{\top}X_0 \ge \beta^{\top}X_1)
\ge 1 - \E[e^{\beta^{\top}(X_0-X_1)/2}]
\ge 1- e^{\psi_0(\frac{1}{2} \bt) + \psi_1(-\frac{1}{2} \bt)}.
$$
A cumulant generating function is convex on its domain.
Maximizing the lower bound over $\beta$ therefore leads to the first-order
condition $\nabla\psi_0(\beta/2) = \nabla\psi_1(-\beta/2)$,
which is the condition in (\ref{psistar}) with $\lambda = 1/2$.
This observation is consistent with the idea that taking $\lambda=1/2$
balances overall performance without emphasizing either specificity or sensitivity
over the other.


\section{Numerical Examples}
\label{s:numerical}
We use simulations to examine the convergence of $\bt_N$ and to illustrate properties of the
classifiers derived using various choices of the penalty function $w$.

\subsection{Convergence Simulations} 

For simplicity, we examine convergence in a one-dimensional example. We have just two
observations from the minority class, $x_1=0$ and $x_2=1$. For the majority class, we
use $N$ i.i.d. samples from the standard normal distribution, $N(0,1)$. We compare
results at several values of $N$.

Table \ref{t:convergence} reports the mean and the standard error of 
the coefficients $(\alpha_N, \bt_N)$ 
for ordinary logistic regression, an exponential objective, 
and an asymptotically linear objective, averaging over 1,000 independent runs.
For the exponential objective we consider $\lambda = 0.5$, 
and for the asymptotically linear objective we use 
$w(u) = -2u+1$, for $u \le 0$, and $w(u) = (u+1)^{-2}$ for $u > 0$,
as in Section~\ref{s:examples}. 

We solve for the optimal coefficients using the \textit{minimize} function in the \textit{scipy.optimizer} 
package with the \textit{Newton-CG} optimization method.
The last row of  Table \ref{t:convergence} reports the limiting value $\bt_*$ 
determined by Theorem~\ref{bdd-thm}, calculated by solving (\ref{bstar:exp}).

\begin{table}[h]
{\small
\centering
\begin{tabular}{r|rr|rr|rr}
\hline
& \multicolumn{2}{|c|}{Logistic $w$} 
& \multicolumn{2}{|c|}{$\lambda = 0.5$}
& \multicolumn{2}{|c}{Linear $w$}  \\
\multicolumn{1}{c|}{$N$} &  \multicolumn{1}{l}{~~~$\alpha_N$} &  \multicolumn{1}{l}{~~$\beta_N$}  & \multicolumn{1}{|l}{~~~$\alpha_N$} &  \multicolumn{1}{l}{~~$\beta_N$} & \multicolumn{1}{|l}{~~~$\alpha_N$} &  \multicolumn{1}{l}{~~$\beta_N$}\\
\hline
10     &  -1.75 (0.047) &  0.86 (0.272) & -1.85 (0.052) &  1.11 (0.372) & -1.77 (0.062) &  1.17 (0.432)  \\
100    &  -4.04 (0.002) &  0.53 (0.014) & -4.17 (0.004) &  0.82 (0.020) & -4.08 (0.003) &  0.63 (0.016)  \\
1,000   &  -6.34 (0.000) &  0.50 (0.001) & -6.48 (0.000) &  0.80 (0.002) & -6.37 (0.000) &  0.57 (0.001)  \\
10,000  &  -8.64 (0.000) &  0.50 (0.000) & -8.78 (0.000) &  0.80 (0.000) & -8.66  (0.000) &  0.55 (0.000) \\
100,000 & -10.95 (0.000) &  0.50 (0.000) & -11.08 (0.000) &  0.80 (0.000) & -10.96 (0.000) &  0.54 (0.000) \\
\hline
True $\bt_*$ & & \multicolumn{1}{l|}{0.50} & & \multicolumn{1}{l|}{0.80} & & \multicolumn{1}{l}{0.50}\\ 
\hline
\end{tabular}}
\caption{Convergence of coefficients as the sample size $N$ of the majority class grows.
Numbers in parentheses are standard errors.}\label{t:convergence}
\end{table}

In all cases, we see that $\alpha_N \to -\infty$ at rate $\log N$, consistent with the 
findings in Corollary~\ref{c:alog} of the appendix. For logistic regression and the exponential objective, 
$\bt_N$ is close to $\bt_*$ at $N = $1,000, when the data is only 0.2\% imbalanced, and 
we have observed a similar convergence rate for other values of $\lambda$.
When the left tail of the weight function diverges linearly, 
we know from Theorem~\ref{bdd-thm} that $\bt_N$ approaches 
the same limit $\bt_*$ as in the case of logistic regression,
but the results in the table indicate that the convergence is much slower. 
We have observed the same behavior with other weight functions whose left tail diverges at a polynomial rate. 

\subsection{High-Sensitivity and High-Specificity Regions}
We turn next to a comparison of logistic regression with exponential objectives at
various values of the exponent $\lambda$. As we discussed in Sections~\ref{s:examples} and~\ref{s:robust2}, 
we expect the value of $\lambda$ to control the relative performance of a classifier 
as measured by sensitivity or specificity.

We consider a two-dimensional example in which $F_0$ is the bivariate standard normal
distribution. Samples from the minority class are drawn from a mixture of two normals: we have
a sample of $n=500$, of which $10\%$ are drawn from $N(\mu_{1,1}, \Sigma_{1,1})$ and $90\%$ from $N(\mu_{1,2}, \Sigma_{1,2})$, with
\[
\mu_{1,1} = \begin{pmatrix}
0\\2
\end{pmatrix}, 
\Sigma_{1,1} = \begin{pmatrix}
0.3 & 0 \\0 & 0.3
\end{pmatrix}, 
\quad
\mu_{1,2} = \begin{pmatrix}
2.3\\2.3
\end{pmatrix}, 
\Sigma_{1,2} = \begin{pmatrix}
0.2 & 0 \\0 & 0.2
\end{pmatrix}.
\]
We will compare results at various values of the sample size $N$ for the majority class.
\begin{figure}[h]
    \centering
    \includegraphics[trim=0.5in 0.75in 0.5in 1in,clip,width=0.5\linewidth]{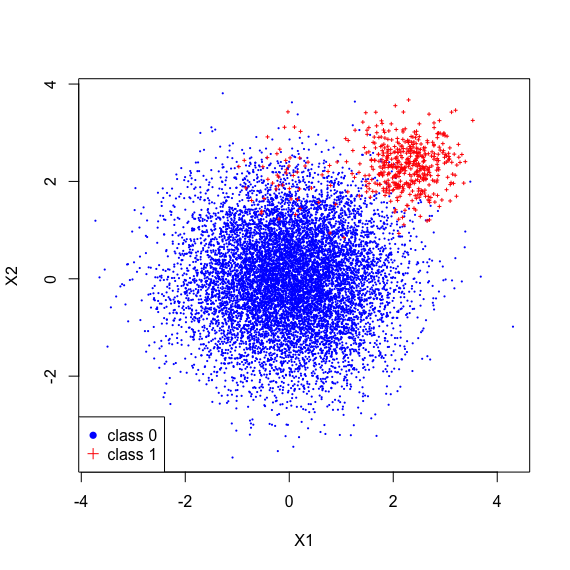}
    \caption{Two-Dimensional Example}
    \label{fig:2d_sim_points}
\end{figure}

Figure \ref{fig:2d_sim_points} shows points drawn from the two classes, with $N =$ 10,000 and $n = 500$. 
The example is designed so that  $90\%$ of the minority class comes from a distribution 
that is easily distinguishable from the majority class, but the remaining $10\%$ makes the classification task challenging.

Recall from Section~\ref{s:ec} that the ROC curve for a linear discriminant function 
$\beta^{\top}x$ is a plot of the true positive rate $\PR_1(\beta^{\top}X>t)$ 
against the false positive rate $\PR_0(\beta^{\top}X>t)$, for all $t\in \R$. 
The area under the ROC, abbreviated AUC, provides an overall summary of the performance 
of the classifier, but we are more
interested in comparing performance at high levels of sensitivity (high true positive rates)
and high levels of specificity (low false positive rates). 
We therefore make comparisons based
on partial AUC (pAUC) measures, as introduced in McClish \cite{mcclish}, for the regions of interest.

The calculation of a specificity-oriented pAUC measure is illustrated in the left panel
of Figure~\ref{fig:pAUC}.
In this example, we focus on the area under the curve between 0 and FP$_1$ (call that ``Area'')
and then normalize it to fall between 1/2 and 1 through the transformation
\[
\mbox{\emph{pAUC}}
= \frac{1}{2}\left(1 + \frac{\text{Area} - \min}{\max - \min}\right)
= \frac{1}{2}\left(1 + \frac{\text{Area} - \text{FP$_1^2/2$}}{\text{FP$_1$} - \text{FP$_1^2/2$}}\right).
\]
Here, max is the area FP$_1$ of the shaded rectangle, and min is the area FP$_1^2/2$
of the triangular portion of the rectangle below the diagonal. 
An ideal classifier over the interval from 0 to FP$_1$ would have a pAUC of 1, whereas
a random assignment of observations to classes would have a pAUC of 1/2.
To focus on high specificity, we consider values of FP$_1$ decreasing from 0.10 to 0,
which corresponds to a true negative rate (TNR) increasing from 0.90 to 1.

The calculation of a sensitivity-oriented pAUC measure on the right side
of Figure~\ref{fig:pAUC} works similarly. 
To focus on high sensitivity, we take a rectangle along the top end of the unit square. 
The lower boundary of that rectangle is defined by a true positive rate (TPR) that we initially set
equal to 0.90 and then increase toward 1. In the normalization of the area for this case,
min is the area to the right of the diagonal. 
We use the R package \emph{pROC} (Robin et al. \cite{robin}) to facilitate the 
calculation and plotting of pAUC values.

\begin{figure}[h]
    \centering
    \includegraphics[width=2.99in]{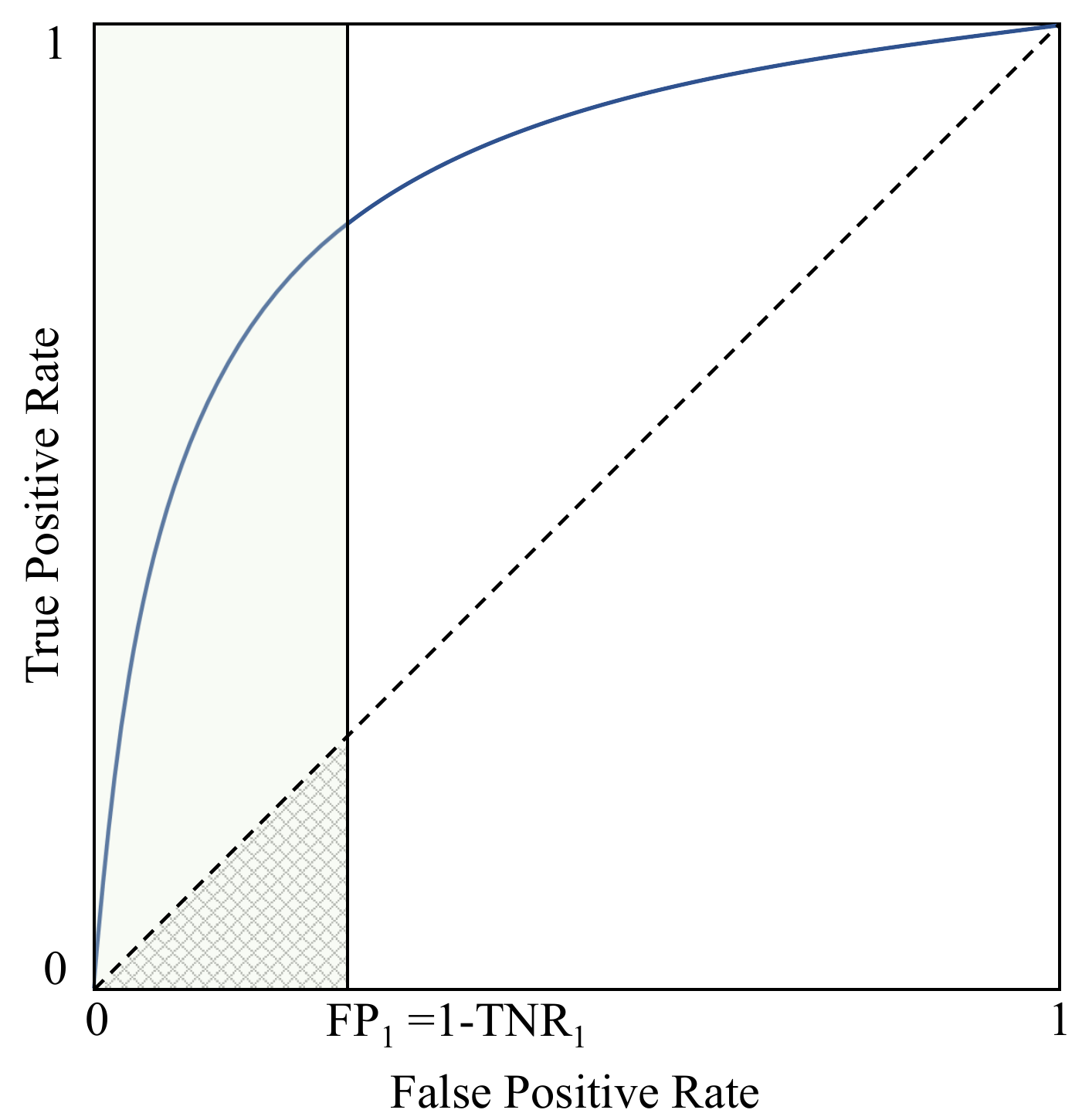}\includegraphics[width=3.10in]{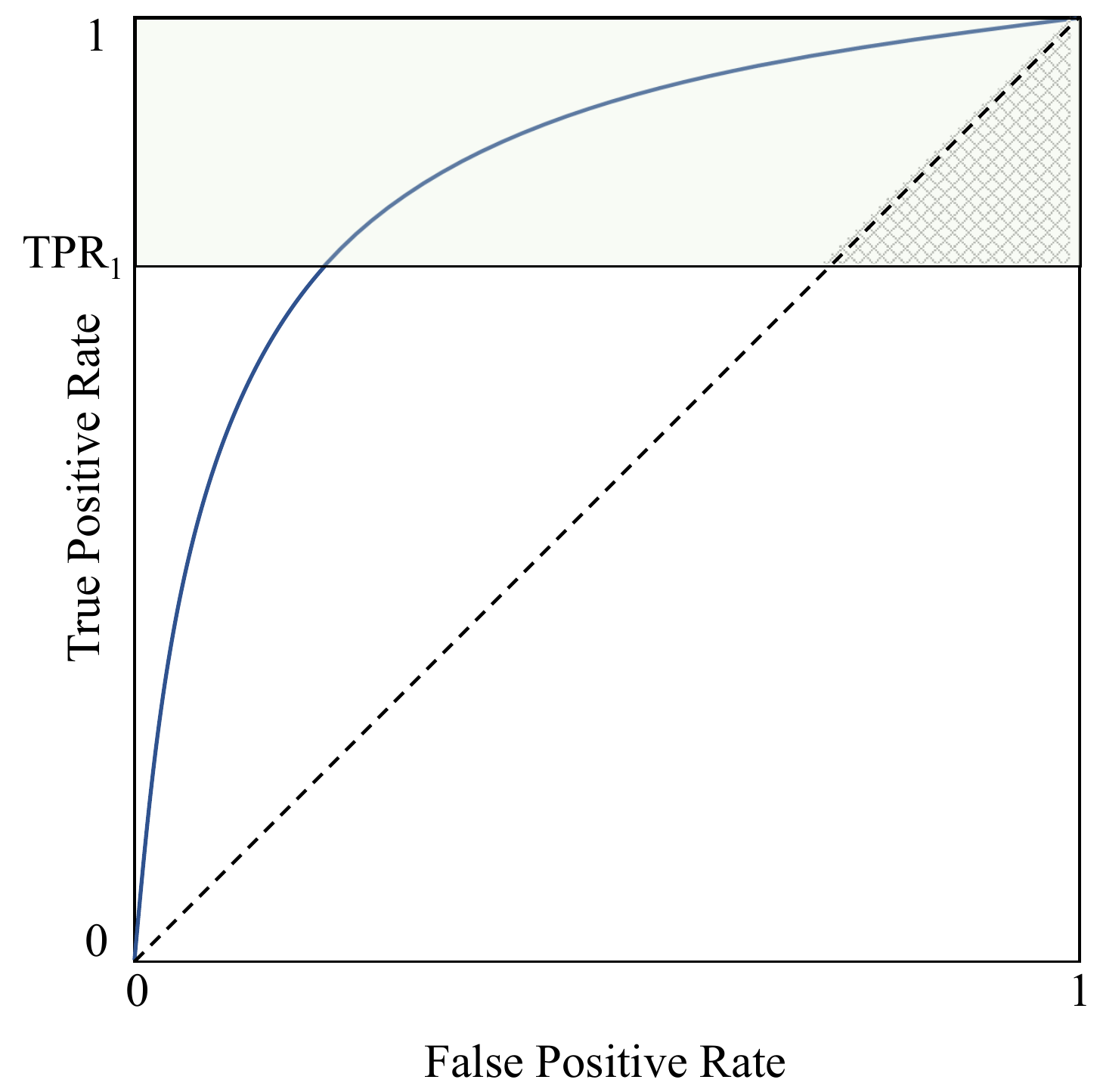}
    \caption{Illustration of a specificity-oriented (left) and sensitivity-oriented (right) pAUC.}
    \label{fig:pAUC}
\end{figure}

Figure \ref{fig:pauc_2d} compares pAUC values for logistic and exponential classifiers with
a sample size of $N=$10,000 for the majority class.
Panel (a) plots pAUCs in the high-sensitivity region, with the true positive rate
$\PR_1(\beta^{\top}X>t)$ increasing from 0.90 toward 1.
Among the exponential
classifiers, we see that, at high levels of sensitivity, the classifier with $\lambda=0.9$
outperforms the classifier with $\lambda=0.5$, which outperforms the classifier with
$\lambda=0.1$. Panel (b) focuses on the region of high specificity, where the true
negative rate $\PR_0(\beta^{\top}X\le t)$ increases from 0.90 toward 1.
Here we see the ordering of the exponential classifiers reversed. This pattern
is consistent with our interpretation of the exponential objective in Section~\ref{s:robust2}: 
higher $\lambda$ puts more weight on sensitivity, and lower $\lambda$ puts more weight on specificity.

We see this pattern as the key consideration in choosing $\lambda$.
In applications such as disease testing or screening for default risk, where a false negative may be much more costly than a false positive, a larger $\lambda$ should be preferred; but if the
goal is to maintain high specificity while optimizing for sensitivity, then a smaller $\lambda$ is more appropriate.
The choice of $\lambda$ does not solve the problem of imbalanced data,
but it helps control the consequences of the imbalance.

At both extremes, Figure \ref{fig:pauc_2d} indicates that the performance of the logistic classifier falls between the exponential classifiers with $\lambda=0.1$ and $\lambda=0.5$.
We investigate this pattern further in Figure~\ref{fig:pauc_varyingN}, where we consider
the effect of a smaller ($N=$1,000) or larger ($N=$50,000) sample size.
Comparing these results with those in Figure \ref{fig:pauc_2d} reveals a consistent pattern:
as $N$ increases, the performance of the logistic classifier becomes indistinguishable
from that of an exponential classifier with small $\lambda$. 
This pattern is consistent with Theorem~\ref{bdd-thm}: if we think of $\bt_*(\lambda)$ as a function of the exponent $\lambda$, then (\ref{bstar:exp}) suggests that $\lim_{\lambda \downarrow 0} \bt_*(\lambda) = \bt_*$, where $\bt_*$ is the limiting coefficient vector for ordinary logistic regression. 

\begin{figure}[h]
    \centering
    \begin{subfigure}[b]{0.475\textwidth}
        \centering
        \includegraphics[width=\textwidth]{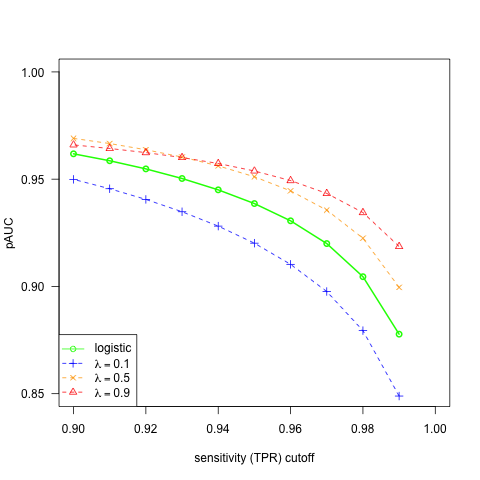}
        \caption[]%
        {{\small High-sensitivity region}}    
    \end{subfigure}
    \hfill
    \begin{subfigure}[b]{0.475\textwidth}  
        \centering 
        \includegraphics[width=\textwidth]{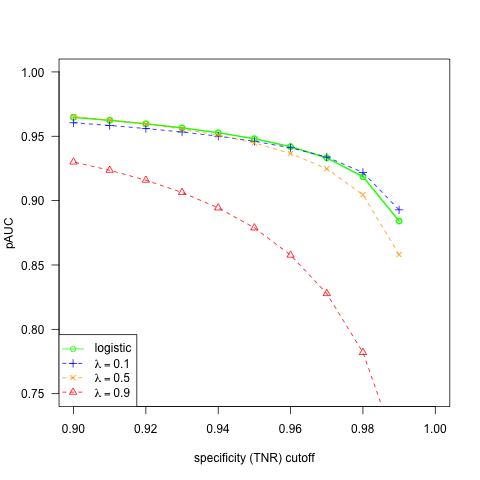}
        \caption[]%
        {{\small High-specificity region}}    
    \end{subfigure}
    \caption{Comparison of pAUC values for logistic and exponential ($\lambda=0.1, 0.5, 0.9$) classifiers}\label{fig:pauc_2d}
\end{figure}

\begin{figure}
    \centering
    \begin{subfigure}[b]{0.475\textwidth}
        \centering
        \includegraphics[width=\textwidth]{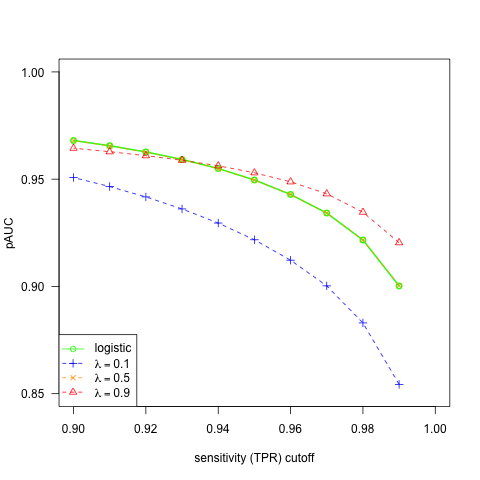}
        \caption{High-sensitivity region, $N = 1,000$}
    \end{subfigure}
    \hfill
    \begin{subfigure}[b]{0.475\textwidth}  
        \centering 
        \includegraphics[width=\textwidth]{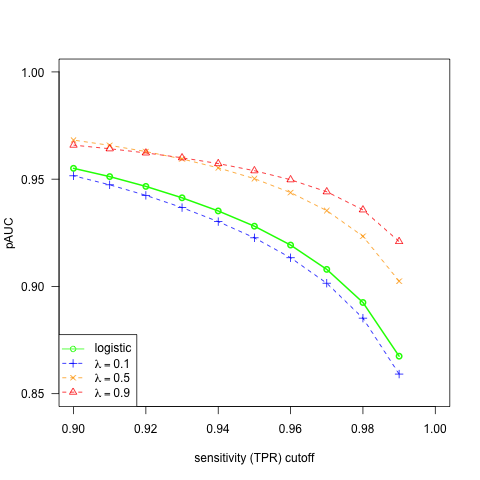}
        \caption{High-sensitivity region, $N = 50,000$}
    \end{subfigure}
    \vskip\baselineskip
    \begin{subfigure}[b]{0.475\textwidth}   
        \centering 
        \includegraphics[width=\textwidth]{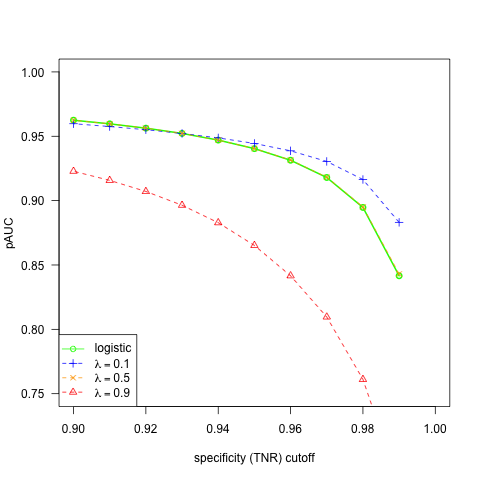}
        \caption{High-specificity region, $N = 1,000$}
    \end{subfigure}
    \hfill
    \begin{subfigure}[b]{0.475\textwidth}   
        \centering 
        \includegraphics[width=\textwidth]{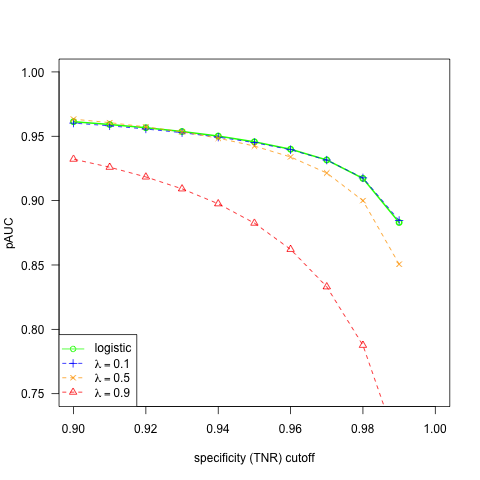}
        \caption{High-specificity region, $N = 50,000$}
    \end{subfigure}
    \caption{As $N$ increases, the logistic pAUC values move closer to the exponential pAUC values with small $\lambda$ in both the high-sensitivity and high-specificity regions}
    \label{fig:pauc_varyingN}
\end{figure}

\section{A Credit Risk Application}
\label{s:credit}

\subsection{Freddie Mac Data}

The task of classifying borrowers by their credit risk is challenged by 
imbalanced data in settings where defaults are rare. In this section, we apply ideas
from previous sections to quarterly data from the
Freddie Mac Single Family Loan-Level Dataset, from $2003$ to $2016$. The dataset can be accessed from \url{http://www.freddiemac.com/research/datasets/sf_loanlevel_dataset.page}.
The dataset covers mortgages purchased or guaranteed by Freddie Mac.

Our outcome of interest --- the binary label we attach to each loan --- is whether the loan defaults within two years of origination. We define a loan to be in default if it is 180 days or more past due. Our goal is to predict this outcome based on loan and borrower features available at origination.
This setup is consistent with Li et al. \cite{yazheLi},
although our sample is much larger.

Figure~\ref{fig:num_loans} plots the number of loans originated in each quarter, and 
Figure~\ref{fig:default_rate} plots the default rate over time from 2003 to 2019.
We exclude from our analysis all loans that were repurchased within two
years of origination.
The default rate is almost always less than $0.03\%$, except around the financial crisis of 2008 when it climbs near $3.5\%$. 
We are thus dealing with extremely imbalanced data and considerable variation in the degree of imbalance.

In predicting outcomes, we use a combination of numerical and categorical attributes.
The numerical variables are credit score, original debt-to-income ratio, 
log of original unpaid principal balance, original loan-to-value ratio, 
and original interest rate; the categorical variables are
number of borrowers (one or more than one), first time homebuyer flag, number of units, occupancy status, loan origination channel, prepayment penalty mortgage flag, property type, and loan purpose.
Precise definitions of these variables can be found in the Freddie Mac \cite{fmac} user guide.

\begin{figure}[h]
\begin{subfigure}{\textwidth}
    \centering
    \caption{Number of Loans} \label{fig:num_loans}
    \includegraphics[width=\linewidth]{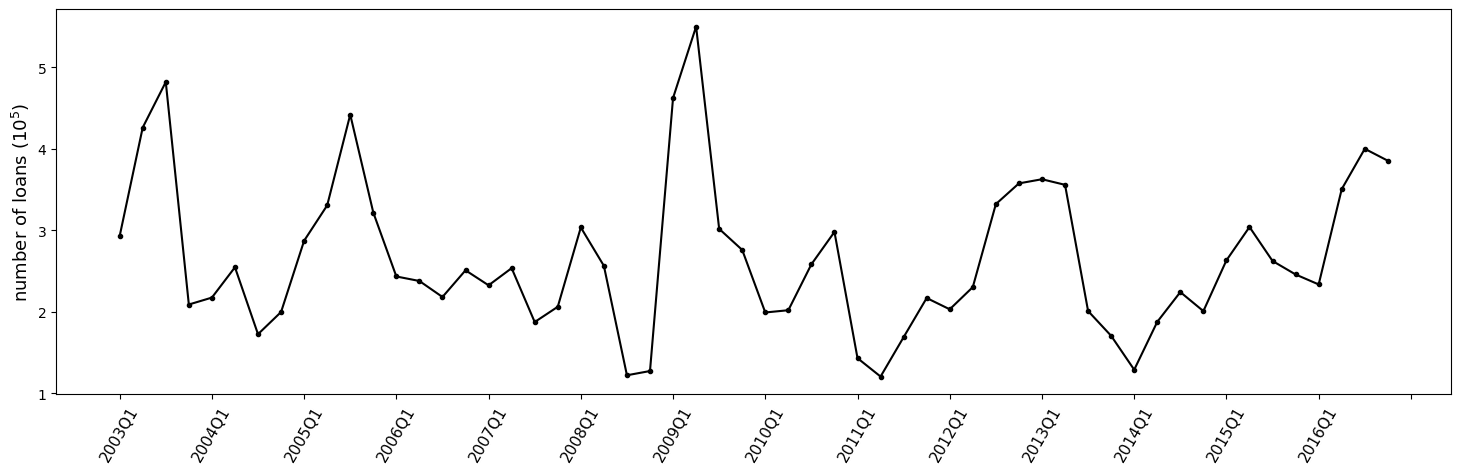}
\end{subfigure}
\bigskip
\\
\begin{subfigure}{\textwidth}
    \centering
    \caption{Default Rate over Time} \label{fig:default_rate}
    \includegraphics[width=\linewidth]{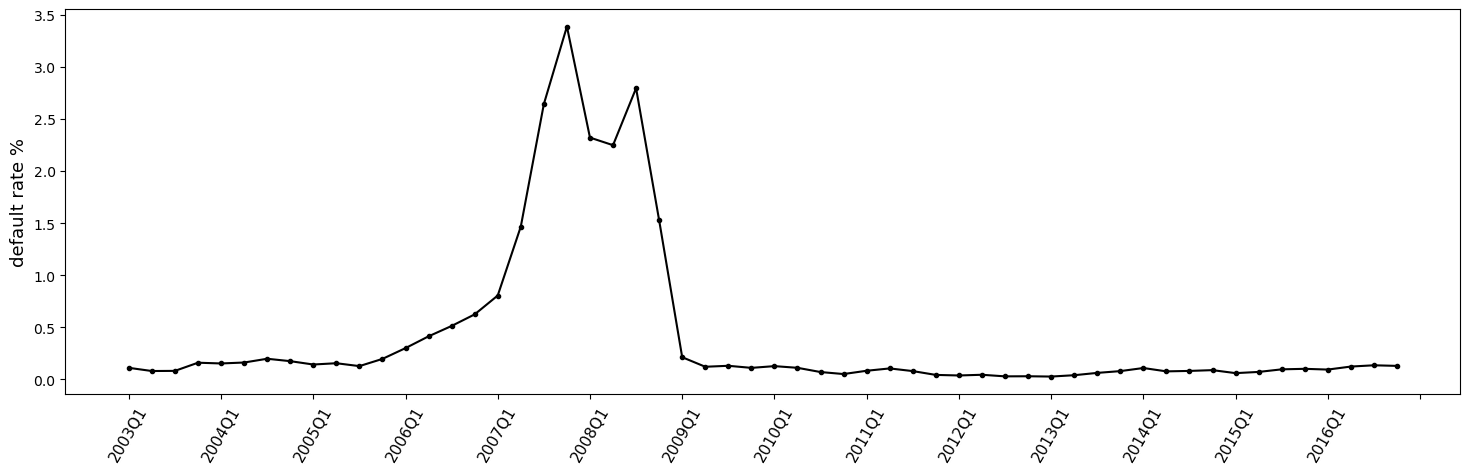}
\end{subfigure}
\caption{Freddie Mac Summary Data}
\end{figure}

We estimate linear classifiers over a rolling window, for $t=2003,\dots,2013$.
For $t=2003$, the process works as follows. We use 80\% of loans originated in any of the four
quarters of 2003 and their default status in the corresponding quarter of 2005 ($t+2$)
to estimate a model, reserving the other 20\% of the data for later validation. This is our training
data for $t=2003$.
We then apply the trained model to the attributes of loans originated in the first quarter
of 2004 to predict default status as of the first quarter of 2006 ($t+3$). This is our
test data for $t=2003$.
We apply the same process, retraining the model with $t=2004$, 
to predict default status in the first quarter of 2007
for loans originated in the first quarter of 2005. Our last forecast is for defaults in
the first quarter of 2016, for loans originated in the first quarter of 2014, trained
based on loans originated in $t=2013$.

We remove loans that are missing values for any numerical variables. For each categorical variable, we interpret missing values as a separate category. 
At each $t$, we check each variable to ensure that  
we have at least two distinct values of the variable in the data
to avoid degeneracy. 
We omit 
the variable for that $t$ if the variable fails this check, 
which happens in fewer than 1\% of cases.

Using this process, we estimate four classifiers at each $t$, using logistic regression and exponential objectives with $\lambda = 0.1, 0.5, 0.9$. 
In Appendix \ref{appendix-fredmac}, we report AUCs for training, validation, and test data
for each classifier, for each $t$. All AUCs are above 0.8, indicating that linear classifiers
perform reasonably well in this task.
The validation AUCs and testing AUCs are all very close to the training AUCs, 
allaying any concerns about overfitting.

\subsection{High-Sensitivity Classifiers}
We consider a lender that would like to apply a simple first-pass classifier that correctly identifies
at least 99\% of customers who would default as high risk. Those classified as high
risk would then undergo a costlier in-depth review. The lender thus wants the
first-pass classifier to have a high TPR to make it highly sensitive
to likely defaulters.
We have seen that, in highly imbalanced settings, logistic regression becomes similar
to an exponential classifier with $\lambda$ close to zero; but we have also seen that
in the high-sensitivity region we should prefer to take $\lambda$ close to one. We investigate
this comparison using the Freddie Mac data.

For each classifier and each year $t$, we set a classification threshold to achieve
a TPR of 99\% in the training data. We then evaluate the TPR and TNR
in the test set for each classifier and each year.
    
Appendix \ref{appendix-testTPR} reports the test TPRs for all methods and all years.
In all cases, the test TPR is close to 99\%, indicating that the threshold set in the training
data works well in the test data.
However, we see clear differences in the test TNRs reported in Table \ref{t:train_TPR_test_TNR}.
In all years, the exponential classifier with $\lambda=0.9$ achieves the best 
or near the best performance, as we expected in this high-sensitivity region. 
The classifier with $\lambda=0.5$ consistently outperforms the logistic classifier and the
case $\lambda=0.1$, and the last two are difficult to distinguish. All of these findings
are consistent with our interpretation of the effect of the parameter $\lambda$ and the
relationship between the logistic and exponential objectives.

\begin{table}[ht]
\centering
\begin{tabular}{rrrrr}
\hline
Year &  Logistic &  $\lambda=0.1$ &  $\lambda =0.5$ &  $\lambda = 0.9$ \\
\hline
       2003 &     29.57 &  29.06 &  33.11 &  \textbf{36.34} \\
       2004 &     24.34 &  24.11 &  \textbf{26.68} &  26.11 \\
       2005 &     23.51 &  23.59 &  \textbf{24.45} &  24.14 \\
       2006 &     22.93 &  23.62 &  26.17 &  \textbf{28.06} \\
       2007 &     21.97 &  21.84 &  23.09 &  \textbf{24.07} \\
       2008 &     23.57 &  23.13 &  24.33 &  \textbf{24.50} \\
       2009 &     29.65 &  25.91 &  29.29 &  \textbf{32.14} \\
       2010 &     17.14 &  17.80 &  22.26 &  \textbf{26.33} \\
       2011 &     31.68 &  30.59 &  29.52 &  \textbf{33.67} \\
       2012 &     33.08 &  32.98 &  35.63 &  \textbf{42.17} \\
       2013 &     26.74 &  26.25 &  23.72 &  \textbf{30.25} \\
\hline
\end{tabular}
\caption{True negative rates (in percent) in test data for classifiers trained at a true positive rate of 99\%}  \label{t:train_TPR_test_TNR}
\end{table}

\subsection{pAUC plots}
To gain further insight into the comparison of the classifiers, we examine pAUC plots
like those introduced in Section~\ref{s:numerical}, but now using the Freddie Mac data.
Figure \ref{fig:pauc_full} shows results for 2007, but we find the same pattern in all years:
as expected, a higher $\lambda$ gives better results in the high-sensitivity region,
and a lower $\lambda$ works better in the high-specificity region.
The performance of logistic regression is similar to that of $\lambda=0.1$ in the first
case but closer to that of $\lambda=0.5$ in the second case due to the finite imbalance in the data.

To compare these linear classifiers with upsampling 
methods often used in practice to address data imbalance, we apply
the SMOTE method of Chawla et al. \cite{smote}.
We apply it using $5$ nearest neighbors, and we upsample the default class to match
the sample size of the non-default class. We then plot the pAUC curve 
of the logistic regression classifier trained on the transformed data
in Figure \ref{fig:pauc_full}.
The results in the figure indicate that SMOTE does not improve
performance in either the high-sensitivity or high-specificity regions.
We have also found that it results in a lower overall AUC than the other methods.

\begin{figure}[p]
    \centering
    \begin{subfigure}[b]{0.475\textwidth}
        \centering
        \includegraphics[width=\textwidth]{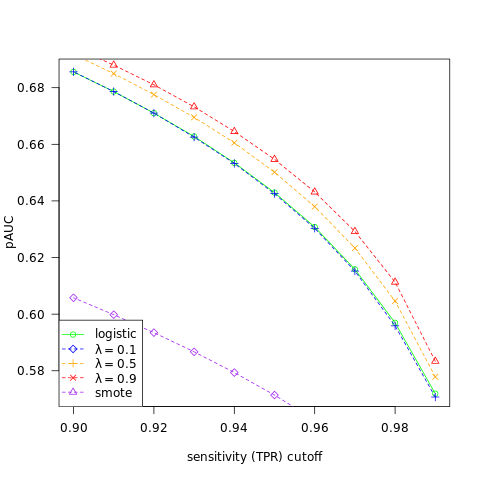}
        \caption[]%
        {{\small High-sensitivity region}}    
    \end{subfigure}
    \hfill
    \begin{subfigure}[b]{0.475\textwidth}  
        \centering 
        \includegraphics[width=\textwidth]{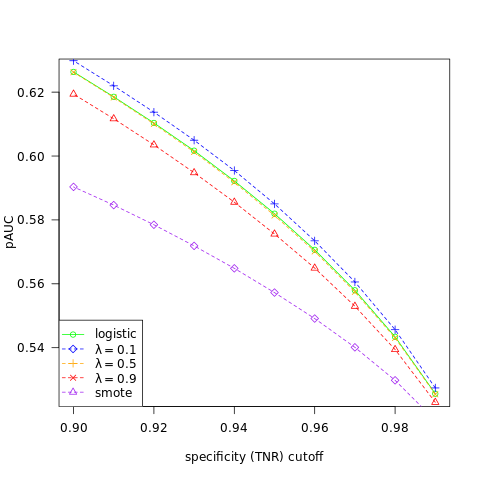}
        \caption[]%
        {{\small High-specificity region}}    
    \end{subfigure}
    \caption{Comparison of pAUC values in test data for logistic and exponential ($\lambda=0.1, 0.5, 0.9$) classifiers using Freddie Mac loan data. SMOTE upsampling with logistic regression is also included for comparison}\label{fig:pauc_full}
\end{figure}

To gauge the statistical significance of differences across $\lambda$ values, Figure \ref{fig:pauc_ci} includes
$90\%$ bootstrap confidence intervals around the pAUC curves. For clarity, we compare
just two cases, $\lambda=0.9$ in red and $\lambda=0.1$ in blue.
The confidence bands barely overlap, indicating that the ordering of the two curves
is reliable.

\begin{figure}[p]
    \centering
    \begin{subfigure}[b]{0.475\textwidth}
        \centering
        \includegraphics[width=\textwidth]{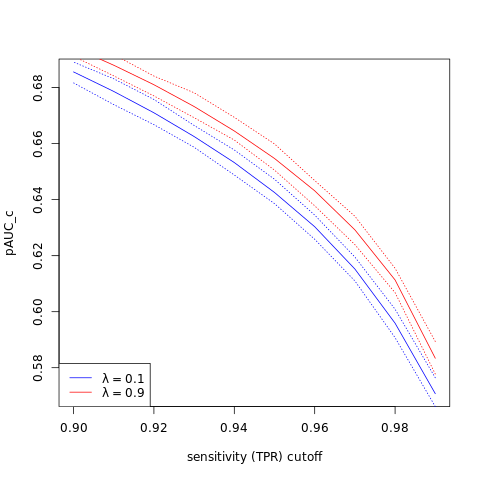}
        \caption[]%
        {{\small High-sensitivity region}}       
    \end{subfigure}
    \hfill
    \begin{subfigure}[b]{0.475\textwidth}  
        \centering 
        \includegraphics[width=\textwidth]{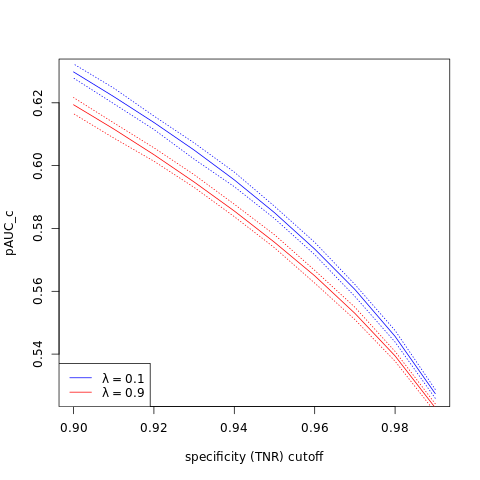}
        \caption[]%
        {{\small High-specificity region}}      
    \end{subfigure}
    \caption{Test data pAUC plots with $90\%$ confidence bands for $\lambda=0.1$ and $\lambda=0.9$.}\label{fig:pauc_ci}
\end{figure}

\section{Concluding Remarks}
\label{s:conc}

We have shown that a broad family of linear discriminant functions have explicit
limits as the sample size of one class grows while the sample size for the other
remains fixed. Linear discriminant functions defined by asymptotically subexponential
weight functions share a common limit with logistic regression.
A wider range of limits applies using asymptotically exponential weights.
The limits of these classifiers reflect different types of robustness or
conservatism towards worst-case false-positive and false-negative
errors. 
Our analysis does not solve the problem of imbalanced data, but it
does provide the user some control over the consequences of
imbalance through the exponent $\lambda$, favoring performance
in either the high-specificity or high-sensitivity regions.
We illustrated these ideas through numerical examples and
an application to credit risk in predicting mortgage defaults.

Our analysis is limited to linear discriminant functions. Linear classifiers are widely
used in practice, at least in part because they are easy to interpret.
Note, also, that the features used for classification could include scores computed from nonlinear models. 
In other words, a linear discriminant function can be used to aggregate results from multiple models.
It would be interesting to know if similar limits hold for scoring rules
derived from regression trees or neural networks. The loss functions
used to optimize these types of rules over parameters are typically nonconvex,
which significantly complicates any analysis of their limiting behavior.


\appendix
\section{Proofs for Section~\ref{s:existence}}

\begin{proof}[Proof of Lemma~\ref{cvx_obj}] Differentiating (\ref{uvint}) yields
\[
\frac{d U}{d u} = w(u), \quad \frac{d^2 U}{d u^2} = w'(u)
\]
and
\[
\frac{d V}{d u} = e^u w(u),\quad \frac{d^2 V}{d u^2} = e^u(w(u) + w'(u)),
\]
so Condition \ref{assump_conv} implies that $U$ is concave and $V$ is strictly convex. 
It follows that each term $-U(\alpha+\beta^{\top}x_i)$ and $V(\alpha+\beta^{\top}X_i)$
is convex in $(\alpha,\beta)$, and thus that $\bar{C}_N$ in (\ref{cost}) is convex.

To establish strict convexity of $\bar{C}_N$, we first claim that, almost surely,
$X_1,\dots,X_N$ do not fall on a hyperplane, for all sufficiently large $N$.
Permuting the order of the $X_i$ does not change whether they fall
on a hyperplane, so we may apply the Hewitt-Savage zero-one law 
(as in, e.g., Durrett \cite{durrett}, p.71)
to conclude that
the probability that all $X_i$ fall on a hyperplane is zero or one.
By the surrounding condition, $F_0$ is not supported on any hyperplane,
so the probability that all $X_i$ fall on a hyperplane is less than one and
must therefore be zero.

Suppose, then, that $N$ is sufficiently large that $X_1,\dots,X_N$ do not fall
on a hyperplane.
Then for any distinct $(\alpha_1,\beta_1)$ and $(\alpha_2,\beta_2)$ there is some
$i_0\in \{1,\dots,N\}$ for which
$\alpha_1+\beta_1^{\top}X_{i_0}\not=\alpha_2+\beta_2^{\top}X_{i_0}$.
Then, for any $\nu\in(0,1)$,
\begin{eqnarray*}
\lefteqn{
\sum_{i=1}^N V(\nu(\alpha_1+\beta_1^{\top}X_i)+(1-\nu)(\alpha_2+\beta_2^{\top}X_i))
} && \\
&=&
V(\nu(\alpha_1+\beta_1^{\top}X_{i_0})+(1-\nu)(\alpha_2+\beta_2^{\top}X_{i_0})) 
+
\sum_{i\not=i_0}^N V(\nu(\alpha_1+\beta_1^{\top}X_i)+(1-\nu)(\alpha_2+\beta_2^{\top}X_i)) \\
&<&
\nu V(\alpha_1+\beta_1^{\top}X_{i_0}) +(1-\nu)V(\alpha_2+\beta_2^{\top}X_{i_0})
+
\sum_{i\not=i_0}^N V(\nu(\alpha_1+\beta_1^{\top}X_i)+(1-\nu)(\alpha_2+\beta_2^{\top}X_i)) \\
&\le &
\nu\sum_{i=1}^N V(\alpha_1+\beta_1^{\top}X_i) +(1-\nu)\sum_{i=1}^N V(\alpha_2+\beta_2^{\top}X_i).
\end{eqnarray*}
The strict inequality follows from the strict convexity of $V$. Strict convexity of $\bar{C}_N$ follows.
\end{proof}

The following result allows us to translate a surrounding condition on $F_0$
to a surrounding condition on its empirical counterparty.

\begin{lemma}\label{l:surround}
Let $\hat{F}_N$ be the empirical distribution of independent
random variables $X_1,\dots,X_N$ drawn from $F$.
Suppose $F$ surrounds $x^*$ with parameters $(\epsilon,\delta)$.
Then for any ${\epsilon}_1<\epsilon$ and ${\delta}_1<\delta$,
$\hat{F}_N$ surrounds $x^*$ with parameters
$({\epsilon}_1,{\delta}_1)$ 
for all sufficiently large $N$, a.s.
\end{lemma}

\begin{proof}[Proof for Lemma~\ref{l:surround}]
For any constants $M>0$ and $\lambda\in(0,1)$, we can choose fixed points
$v_1,\dots,v_K\in \Omega$, with $K$ depending on $M$ and $\lambda$, such
that, for every $\omega\in \Omega$,
\begin{equation}
\min_{k=1,\dots,K}\|\omega - v_k\| \le \frac{\lambda\epsilon}{M}.
\label{vks}
\end{equation}
This follows from the relative compactness of $\Omega$.
For any $0<\delta'<\delta$, we may take $M$ sufficiently large that
$$
\PR(\|X-x^*\| > M) < \delta',
$$
with $X$ having distribution $F$.

It follows from (\ref{vks}) that
for any sequence $\omega_N\in\Omega$ we may choose
$k_N\in\{1,\dots,K\}$ such that, for all $N$,
$$
\|\omega_N - v_{k_N}\| \le \frac{\lambda\epsilon}{M}.
$$
The sequences $\omega_N$ and $k_N$ may be stochastic.
For any $x\in\R^d$,
\begin{eqnarray*}
\1\{(x-x^*)^{\top}v_{k_N} > \epsilon\}
&= &
\1\{(x-x^*)^{\top}\omega_{N} + (x-x^*)^{\top}(v_{k_N}-\omega_{N})> \epsilon\} \\
&\le  &
\1\{(x-x^*)^{\top}\omega_{N} + \|x-x^*\|\lambda\epsilon/M> \epsilon\} \\
&\le  &
\1\{(x-x^*)^{\top}\omega_{N} >(1-\lambda)\epsilon \}+ \1\{\|x-x^*\|\lambda\epsilon/M> \lambda\epsilon\} \\
&=  &
\1\{(x-x^*)^{\top}\omega_{N} >(1-\lambda)\epsilon \}+ \1\{\|x-x^*\| >M\}. 
\end{eqnarray*}
Thus, for any $i=1,\dots,N$,
$$
\1\{(X_i-x^*)^{\top}\omega_{N} >(1-\lambda)\epsilon \}
\ge
\1\{(X_i-x^*)^{\top}v_{k_N} >\epsilon \}
-\1\{\|X_i-x^*\|>M\},
$$
a.s., and also, a.s.,
\begin{eqnarray}
\lefteqn{
\frac{1}{{N}}\sum_{i=1}^{N}\1\{(X_i-x^*)^{\top}\omega_{N} >(1-\lambda)\epsilon \}
\ge} && \nonumber \\
&&
\frac{1}{{N}}\sum_{i=1}^{N}\1\{(X_i-x^*)^{\top}v_{k_N} >\epsilon \}
-\frac{1}{{N}}\sum_{i=1}^{N}\1\{\|X_i-x^*\|>M\}.
\label{fracbnd}
\end{eqnarray}
For the first term on the right, note that the strong law of large numbers 
and the surrounding condition
for $F$ imply that, for each $k=1,\dots,K$, we have the almost sure limit
$$
\frac{1}{{N}}\sum_{i=1}^{N}\1\{(X_i-x^*)^{\top}v_k >\epsilon \}
\to \PR((X-x^*)^{\top}v_k >\epsilon) > \delta,
$$
so 
$$
\liminf_{N\to\infty}
\frac{1}{{N}}\sum_{i=1}^{N}\1\{(X_i-x^*)^{\top}v_{k_N} >\epsilon \}
> \delta,\quad \mbox{a.s.}
$$
For the second term on the right side of (\ref{fracbnd}), we have, a.s.,
$$
\frac{1}{{N}}\sum_{i=1}^{N}\1\{\|X_i-x^*\|>M\}
\to \PR(\|X-x^*\|>M) < \delta',
$$
by the strong law of large numbers and our choice of $M$.
Thus, (\ref{fracbnd}) yields
$$
\liminf_{N\to\infty}
\frac{1}{{N}}\sum_{i=1}^{N}\1\{(X_i-x^*)^{\top}\omega_{N} >(1-\lambda)\epsilon \}
>\delta-\delta', \quad \mbox{a.s.}
$$
This implies
$$
\liminf_{N\to\infty}
\min_{\omega\in\Omega}\frac{1}{{N}}\sum_{i=1}^{N}\1\{(X_i-x^*)^{\top}\omega >(1-\lambda)\epsilon \}
>\delta - \delta', \quad \mbox{a.s.}
$$
because for each $N$ the sum takes only finitely many values as $\omega$ varies, 
so the minimum over $\omega$ is attained, and we may take $\omega_N$
to be the minimizing $\omega$.
It follows that,
$$
\min_{\omega\in\Omega}\frac{1}{{N}}\sum_{i=1}^{N}\1\{(X_i-x^*)^{\top}\omega >(1-\lambda)\epsilon \}
>\delta-\delta',
$$
for all sufficiently large $N$, a.s. As $\lambda$ 
and $\delta'$ may be arbitrarily close to zero, the result follows.
\end{proof}

\begin{proof}[Proof of Lemma~\ref{existence}]
We know from Lemma~\ref{cvx_obj} that $\bar{C}_N$ is strictly convex in $(\alpha,\beta)$
for all sufficiently large $N$, a.s.
Strict convexity implies that either $\bar{C}_N$ has a unique minimizer or it
strictly decreases along some ray $\{(\lambda \alpha_0, \lambda\beta_0) | 0 \le \lambda < \infty\}$,
with $(\alpha_0,\beta_0)$ not identically zero. 
We will show that the latter case is not possible.

We treat separately the cases $\bt_0 = 0$ (with $\alpha_0\not=0$)
and $\bt_0^\top \bt_0 = 1$. The normalization in the second case is justified by our scaling
by $\lambda$.

Case 1: $\bt_0 = 0$, $\alpha_0\not=0$. Differentiation yields
\begin{eqnarray*}
\frac{\partial \bar{C}_N(\lambda \alpha_0, \lambda \bt_0)}{\partial \lambda} 
&=&
-\sum_{i=1}^n \frac{\partial}{\partial\lambda} U(\lambda\alpha_0)
+\sum_{i=1}^N \frac{\partial}{\partial\lambda} V(\lambda\alpha_0) \\
&=& - nw(\lambda\alpha_0)\alpha_0 + Ne^{\lambda \alpha_0} w(\lambda \alpha_0) \alpha_0 \\
&=& (-n + N e^{\lambda \alpha_0}) w(\lambda \alpha_0) \alpha_0.
\end{eqnarray*}
If $\alpha_0 > 0$, then for $\lambda$ large, $e^{\lambda \alpha_0} > n/N$, and the derivative is strictly positive. If $\alpha_0 < 0$, then for $\lambda$ large, $e^{\lambda \alpha_0} < n/N$ and 
the derivative is again strictly positive. 

Case 2: $\bt_0^\top \bt_0 = 1$. Differentiation yields
\begin{eqnarray}
		\frac{\partial \bar{C}_N(\lambda \alpha_0, \lambda \bt_0)}{\partial \lambda} 
		&=& 
		\sum_{i=1}^n  - w(\lambda \alpha_0 + \lambda \bt_0^\top x_i) (\alpha_0 + \bt_0^\top x_i) \nonumber \\
		&&+
		\sum_{i=1}^N e^{\lambda \alpha_0 + \lambda \bt_0^\top X_i} w(\lambda \alpha_0 + \lambda \bt_0^\top X_i) (\alpha_0 + \bt_0^\top X_i)
		\nonumber \\
 		&=& 
 		\sum_{i: \alpha_0 + \bt_0^\top x_i < 0} - w(\lambda \alpha_0 + \lambda \bt_0^\top x_i) (\alpha_0 + \bt_0^\top x_i) \label{t1} \\
 		&&+  
 		\sum_{i: \alpha_0 + \bt_0^\top x_i > 0} - w(\lambda \alpha_0 + \lambda \bt_0^\top x_i) (\alpha_0 + \bt_0^\top x_i) \label{t2} \\
 		&&+  
 		\sum_{i:\alpha_0 + \bt_0^\top X_i < 0} e^{\lambda \alpha_0 + \lambda \bt_0^\top X_i} w(\lambda \alpha_0 + \lambda \bt_0^\top X_i) (\alpha_0 + \bt_0^\top X_i) \label{t3} \\
 		&&+  
 		\sum_{i:\alpha_0 + \bt_0^\top X_i > 0} e^{\lambda \alpha_0 + \lambda \bt_0^\top X_i} w(\lambda \alpha_0 + \lambda \bt_0^\top X_i) (\alpha_0 + \bt_0^\top X_i). \label{t4}
\end{eqnarray}
We will show that as $\lambda$ increases, the liminf of the derivative on the left is strictly
positive. We will prove this by showing that (\ref{t2}) and (\ref{t3}) approach zero as
$\lambda$ increases, and the sum of (\ref{t1}) and (\ref{t4}) remains positive and
bounded away from zero as $\lambda$ increases.

Recall from the comments after Condition~\ref{assump_basic} that
$w(s)\to 0$ and $e^{-s}w(-s)\to 0$ as $s\to\infty$.
Thus, in (\ref{t2}), $\alpha_0+\beta_0^{\top}x_i>0$ implies $\lambda(\alpha_0+\beta_0^{\top}x_i)
\to \infty$, and $w(\lambda\alpha_0+\lambda\beta_0^{\top}x_i)\to 0$.
In (\ref{t3}), $\alpha_0 + \bt_0^\top X_i < 0$ implies
$e^{\lambda \alpha_0 + \lambda \bt_0^\top X_i} w(\lambda \alpha_0 + \lambda \bt_0^\top X_i)\to 0$. Thus, (\ref{t2}) and (\ref{t3}) approach zero as
$\lambda$ increases.

The terms in (\ref{t1}) and (\ref{t4}) are nonnegative. We need to show that at least one
of them remains bounded away from zero. Using the fact that $w(s)>0$ is decreasing
we get a lower bound for (\ref{t1}),
\begin{eqnarray}
\sum_{i: \alpha_0 + \bt_0^\top x_i < 0} - w(\lambda \alpha_0 + \lambda \bt_0^\top x_i) (\alpha_0 + \bt_0^\top x_i) 
&\ge&
-w(0)\sum_{i: \alpha_0 + \bt_0^\top x_i < 0} (\alpha_0 + \bt_0^\top x_i)  \nonumber \\
&\ge&
-w(0)\sum_{i=1}^n (\alpha_0 + \bt_0^\top x_i)  \nonumber \\
&=& -n\cdot w(0)(\alpha_0 + \beta_0^{\top}\bar{x}).
\label{lb1}
\end{eqnarray}
Using the fact that $e^sw(s)$ is increasing we get a lower bound for (\ref{t4}),
\begin{eqnarray}
\lefteqn{\hspace*{-1.5in}
\sum_{i:\alpha_0 + \bt_0^\top X_i > 0} e^{\lambda \alpha_0 + \lambda \bt_0^\top X_i} w(\lambda \alpha_0 + \lambda \bt_0^\top X_i) (\alpha_0 + \bt_0^\top X_i) } && \nonumber \\
&\ge&
w(0)\sum_{i=1}^N\1\{\alpha_0 + \bt_0^\top X_i > 0\}(\alpha_0 + \bt_0^\top X_i) \nonumber \\
&\ge&
w(0)\epsilon_1 \sum_{i=1}^N\1\{\alpha_0 + \bt_0^\top X_i > \epsilon_1\}  \nonumber \\
&=&
w(0)\epsilon_1\sum_{i=1}^N  \1\{\bt_0^\top (X_i-\bar{x}) > \epsilon_1-(\alpha_0+\beta_0^{\top}\bar{x})\}.
\label{lb2}
\end{eqnarray}
In light of Lemma~\ref{l:surround}, we may suppose $N$ is sufficiently large
that the empirical distribution of $X_1,\dots,X_N$
surrounds $\bar{x}$.
If $\alpha_0+\beta_0^{\top}\bar{x}<0$, then (\ref{lb1}) is strictly positive;
if $\alpha_0+\beta_0^{\top}\bar{x}\ge 0$, then the surrounding condition implies
that (\ref{lb2}) is strictly positive, for sufficiently small $\epsilon_1>0$.
\end{proof}

\section{Proofs for Section~\ref{s:main}}
\subsection{Proof of Lemma~\ref{l:w_examples}}
\begin{proof}[Proof of Lemma~\ref{l:w_examples}]
To show that the three examples of functions satisfy the left-tail conditions in Definition~\ref{d:wti}, we need to show that for any $k\ge0$ and $C>0$, $h(u) = C|u|^k$ 
satisfies the requirements on $h$ in Definition~\ref{d:wti}. Then we can choose $\lambda = k=0$; or
$\lambda=0$ and $k>0$; or $\lambda\in(0,1)$ and $k>0$ for the three cases of weight functions.
    
To see why $h(u) = C|u|^k$ satisfies the requirements on $h$ in Definition~\ref{d:wti}, 
we note that for $u < 0$, $h'(u) = -C k |u|^{k-1}$, and 
\[
\liminf_{u \to -\infty} h'(u)/h(u) = \liminf_{u \to -\infty} -k/|u| = 0,
\]
so left-tail condition (iii) in Definition~\ref{d:wti} is satisfied. Now notice that
$$
\frac{h(u+s)}{h(u)} = \left |1 + \frac{s}{u} \right |^k.
$$
Let $\epsilon > 0$. We may find $u_1 < 0$ such that for any $u \le u_1$ and any $|s| \le 1$,
\begin{equation}\label{l-bd-poly}
    1 - \epsilon \le \left | 1 + \frac{s}{u} \right |^k \le
    1 + \epsilon.
\end{equation}
We may find $u_2 < 0$ such that for any $u \le u_2$ and any $|s| > 1$,
\begin{equation}\label{u-bd-poly}
1 - \epsilon |s|^k \le \left | 1 + \frac{s}{u} \right |^k 
\le
1 + \epsilon |s|^k.
\end{equation}
Combining (\ref{l-bd-poly}) and (\ref{u-bd-poly}), we have for $u \le \min\{u_1, u_2\}$ and for any $s$,
\begin{equation}\label{ul-bd-poly}
   1 -\epsilon \max\{1, |s|^k \} 
\le \left | 1 + \frac{s}{u} \right |^k 
\le 1 + \epsilon \max\{1, |s|^k \}.
\end{equation}
That is,
\[
\left | \left | 1 + \frac{s}{u} \right |^k - 1 \right | \le \epsilon \max\{1, |s|^k \},
\]
so (\ref{h-rate}) is satisfied by any $C>1$, $\xi>0$, $s_0>1$ such that
$C\ge \max\{1,|s|^k\}$ for $|s|<s_0$, and $e^{\xi|s|}\ge |s|^k$
for $|s|\ge s_0$.
\end{proof}

\subsection{Proof of Proposition~\ref{prop:id}}
\begin{proof}[Proof of Proposition~\ref{prop:id}]
Suppose $w(u) \sim e^{-\lambda u} h(u)$, where $\lambda$ and $h(u)$ satisfy the conditions in Definition~\ref{d:wti}. 
Suppose $\tilde{\lambda}$ and $\tilde{h}$ also satisfy the conditions in Definition~\ref{d:wti} 
and $w(u) \sim e^{-\tilde{\lambda} u} \tilde{h}(u)$.

Let $\Delta = \tilde{\lambda} - \lambda$. Then $w(u) \sim e^{-(\lambda+\Delta) u} \tilde{h}(u)$. 
Since $w(u) \sim e^{-\lambda u} h(u)$, we must have $\tilde{h}(u) \sim e^{\Delta u} h(u) $. 
We will show that $\tilde{h}$ fails left-tail condition (iv) in Definition~\ref{d:wti} unless $\Delta = 0$.

By (\ref{h-rate}), for any $s \in \R$,
$$
\frac{h(u+s)}{h(u)} \to 1,
$$
as $u \to -\infty$. Therefore,
$$
\lim_{u\to-\infty}
 \frac{\tilde{h}(u+s)}{\tilde{h}(u)} - 1 =   
 \lim_{u\to-\infty}\frac{h(u+s)}{h(u)} e^{\Delta s} - 1 = e^{\Delta s} - 1,
$$
and for $s \ne 0$,
$$
\left | \frac{\tilde{h}(u+s)}{\tilde{h}(u)} - 1 \right | \not\to 0
$$
violating (\ref{h-rate}) unless $\Delta = 0$. 
Thus, (\ref{h-rate}) requires $\tilde{\lambda} = \lambda$ and $\tilde{h}(u) \sim h(u)$ as $u \to -\infty$.
\end{proof}

\subsection{A Convergence Result}

The following proposition is key to our main result.

\begin{proposition}\label{p:abi}
Suppose Conditions \ref{assump_basic}--\ref{assump_surr_all_min} hold, and suppose $w$ satisfies Definition~\ref{d:wti} with $w(u) \sim e^{-\lambda u}h(u)$ as $u \to -\infty$.
Let $\gamma = (1-\lambda) \epsilon\delta$, where $\epsilon,\delta>0$
are the surrounding parameters in Condition \ref{assump_surr_mean}. 
Then, almost surely,
\begin{equation}
\alpha_N\to - \infty \quad \mbox{and}\quad \limsup_{N\to\infty}\norm{\bt_N}
\le 1/\gamma.
\label{ablimsi}
\end{equation}
\end{proposition}

We separate the proof into two steps, first showing that 
$\alpha_N+\beta^{\top}_N\bar{x}\to -\infty$, a.s., and then showing (\ref{ablimsi}). 

\begin{lemma}[Step 1]\label{l:step1i}
Under the conditions of Proposition~\ref{p:abi}, $\alpha_N+\beta^{\top}_N\bar{x}\to -\infty$, a.s.
\end{lemma}

We will prove two cases separately: (i) bounded weight functions with $\lambda = 0$ and $h(u) \equiv C > 0$ in 
Definition~\ref{d:wti}, and (ii) unbounded weight functions with non-constant $h$ or $\lambda > 0$. 

\begin{proof}[Proof of Lemma~\ref{l:step1i} for bounded weight functions]
Recalling that $w$ is decreasing, let $C = \lim_{u \to -\infty} w(u)$ and
let $(\epsilon, \delta)$ be the parameters in Condition~\ref{assump_surr_mean}. 
Then at any $(\alpha,\beta)$ and for any $\delta_1 \in (0,\delta)$,
\begin{eqnarray}
	\frac{\partial \bar{C}_N}{\partial \alpha} 
	&=& \sum_{i=1}^n -w(\si) + \sum_{i=1}^N e^{\alpha+\beta^{\top}X_i} w(\alpha+\beta^{\top}X_i) \nonumber \\
	&=& \sum_{i=1}^n -w(\si) + \sum_{i=1}^N e^{\alpha + \bt^\top \bar{x} + \bt^\top (X_i - \bar{x})} w(\alpha + \bt^\top \bar{x} + \bt^\top (X_i - \bar{x}))\nonumber\\
	&\ge& -n C + e^{\alpha + \bt^\top \bar{x}} w(\alpha + \bt^\top \bar{x}) 
	\sum_{i=1}^N \1\{\bt^\top (X_i - \bar{x}) \ge 0\}  \nonumber\\
	&\ge& -n C + N e^{\alpha + \bt^\top \bar{x}} w(\alpha + \bt^\top \bar{x}) 
	\delta_1,
\label{dcbnd}
\end{eqnarray}
for all sufficiently large $N$, a.s., 
where going from the second to the third line we used the conditions that $w(u)$ is decreasing and $w(u) e^u$ is increasing, 
and going from the third to the fourth line we applied Lemma~\ref{l:surround}.

Consider any $(\alpha,\beta)$ for which $e^{\alpha + \bt^\top \bar{x}}w(\alpha + \bt^\top \bar{x}) >  n C/(N \delta_1)$. 
For any such $(\alpha,\beta)$, (\ref{dcbnd}) implies that $\partial \bar{C}_N/\partial \alpha > 0$.
It follows that no such $(\alpha,\beta)$ can
be optimal; the optimal $(\alpha_N,\beta_N)$
must satisfy the reverse inequality
\begin{equation}
e^{\alpha_N + \bt_N^\top \bar{x}}w(\alpha_N + \bt^\top_N \bar{x})  \le nC/(N\delta_1),
\label{anbnx}
\end{equation}
from which we get $\alpha_N + \bt_N^\top \bar{x} \to -\infty$, a.s., which completes
the proof for the bounded case.
\end{proof}

To prove Lemma~\ref{l:step1i} for unbounded weight functions, we will
need the following result. (The following result also holds for bounded weight
functions, which will be useful in Corollary~\ref{c:alog}.)

\begin{lemma}\label{lemma:min_div}
Suppose Conditions \ref{assump_basic}, \ref{assump_conv}, and \ref{assump_surr_all_min} hold. Then
\begin{equation}
\min_i \alpha_N + \bt_N^\top x_i \to -\infty.
\label{minlim}
\end{equation}
\end{lemma}

\begin{proof}[Proof of Lemma~\ref{lemma:min_div}]
Let $(\epsilon^o, \delta^o)$ be the surrounding parameters in Condition~\ref{assump_surr_all_min}. Then for any $j=1,\dots,n$ at any $(\alpha,\beta)$
and for any $\delta_1 \in (0,\delta^o)$,
\begin{eqnarray}
	\frac{\partial \bar{C}_N}{\partial \alpha} 
	&=& \sum_{i=1}^n -w(\si) + \sum_{i=1}^N e^{\alpha+\beta^{\top}X_i} w(\alpha+\beta^{\top}X_i) \nonumber \\
	&=& \sum_{i=1}^n -w(\si) + \sum_{i=1}^N e^{\alpha + \bt^\top x_j + \bt^\top (X_i - x_j)} w(\alpha + \bt^\top x_j + \bt^\top (X_i - x_j))\nonumber\\
	&\ge& \sum_{i=1}^n -w(\si) + e^{\alpha + \bt^\top x_j} w(\alpha + \bt^\top x_j) \sum_{i=1}^N \1\{\bt^\top (X_i - x_j) \ge 0\}  \nonumber\\
	&\ge& -n \max_i w(\alpha + \bt^\top x_i) + N e^{\alpha + \bt^\top x_j} w(\alpha + \bt^\top x_j) \delta_1^o,
\label{dcbnd}
\end{eqnarray}
for all sufficiently large $N$, a.s., where going from the second to the third line we used the conditions that $w(u)$ is decreasing and $w(u) e^u$ is increasing, and going from the third to the fourth line we applied Lemma~\ref{l:surround}.

Let $j(N) \in \mathrm{argmin}_{i=1,2,\dots,n} \{\alpha_N + \bt_N^\top x_i\}$. Because 
$-w(s)$ and $e^sw(s)$ are monotonically increasing, at the minimizer $(\alpha_N,\beta_N)$ we have
\begin{equation}\label{eqn:poly_rate}
\begin{aligned}
0=
\frac{\partial \bar{C}_N}{\partial \alpha} (\alpha_N, \bt_N)
	&\ge -n w(\alpha_N + \bt_N^\top x_{j(N)}) + N e^{\alpha_N + \bt_N^\top x_{j(N)}} w(\alpha_N + \bt_N^\top x_{j(N)}) \delta_1^o \\
	&= w(\alpha_N + \bt_N^\top x_{j(N)}) (-n + N e^{\alpha_N + \bt_N^\top x_{j(N)}} \delta_1^o),
\end{aligned}
\end{equation}
so, $\alpha_N + \bt_N^\top x_{j(N)} \to -\infty$ a.s., which is (\ref{minlim}).
\end{proof}

We can now prove Lemma~\ref{l:step1i} for unbounded $w$.
Recall that for unbounded $w$ we require the right-tail condition in
Definition~\ref{d:wti}.

\begin{proof}[Proof of Lemma~\ref{l:step1i} for unbounded weight functions.]
Let $(\epsilon^o, \delta^o)$ be the surrounding parameters in Condition~\ref{assump_surr_all_min}.
We now use (\ref{minlim}) to show that $\alpha_N + \beta_N^{\top}\bar{x}\to -\infty$ almost surely.

For $j=1,\dots,n$, we introduce the centered loss (centered around $x_j$)
\begin{equation}\label{centerCost_j}
C^j(\alpha,\bt) = \sum_{i=1}^{n} -U(\alpha + \bt^\top(x_i - x_j)) + \sum_{k=1}^N V(\alpha +\bt^\top (X_k-x_j)).
\end{equation}
With $(\alpha_N,\beta_N)$ the minimizer of $\bar{C}_N$, the centered loss $C^j$ is minimized
at $(\alpha^j_N, \bt^j_N)$, where
\[
\alpha^j_N = \alpha_N + \bt_N^\top x_j,\quad  \bt^j_N = \bt_N.
\]
Consider 
\begin{equation}
	\begin{aligned}
		&C^j(\alpha, 0) - C^j(\alpha, \bt) \\
		&= \sum_{i=1}^n [-U(\alpha) + U(\alpha+\bt^\top (x_i-x_j))] + N V(\alpha) - \sum_{k=1}^N V(\alpha + \bt^\top (X_k-x_j)).
	\end{aligned}
	\label{dcj}
\end{equation}
Since $U$ is concave with $dU/du = w(u)$,
\[
U(\alpha+\bt^\top (x_i-x_j)) \le U(\alpha) + w(\alpha) \bt^\top (x_i-x_j)
\]
and
\begin{equation}
     \sum_{i=1}^n [-U(\alpha) + U(\alpha+\bt^\top (x_i-x_j))] \le \sum_{i=1}^n w(\alpha) \bt^\top (x_i-x_j) \le n w(\alpha) \norm{\bt} C,
\label{ubnd}
\end{equation}
where $C = \max_{i,j} \norm{x_i - x_j}$.

Similarly, $V$ is convex and strictly positive with $dV/du = e^uw(u)$, so,
for $x\not=x_j$,
\begin{equation}
V(\alpha + \bt^\top (x-x_j)) \ge [V(\alpha) + e^\alpha w(\alpha) \bt^\top (x-x_j)]_+ 
\ge e^\alpha w(\alpha) [\bt^\top (x-x_j)]_+.
\label{vbnd}
\end{equation}
By Condition~\ref{assump_surr_all_min} and Lemma~\ref{l:surround}, 
the empirical distribution of $X_1, \dots, X_N$ surrounds all $x_i$, $i=1,2\dots,n,$ for any parameters $\epsilon_1^o \in (0,\epsilon^o)$ and $\delta_1^o \in (0, \delta^o)$, for all sufficiently large $N$ a.s. Therefore,
\[
\min_i \inf_{\omega \in \Omega} \sum_{k=1}^N [(X_k-x_i)^\top \omega]_+ \ge \min_i \inf_{\omega \in \Omega} \sum_{k:(X_k - x_i)^\top \omega > \epsilon_1^o}^N [(X_k-x_i)^\top \omega]_+ \ge \epsilon_1^o \delta_1^o \equiv \gamma_1^o > 0
\]
where $\Omega = \{\omega \in \R^d | \omega^\top \omega = 1\}$.
Applying this bound with (\ref{vbnd}) we get
\begin{equation}
		- \sum_{k=1}^N V(\alpha + \bt^\top (X_k-x_j)) \le - e^\alpha w(\alpha) \sum_{k=1}^N [\bt^\top (X_k-x_j)]_+ \le - N e^\alpha w(\alpha) \norm{\bt} \gamma_1^o.
\label{ivbnd}
\end{equation}
Applying (\ref{ubnd}) and (\ref{ivbnd}) in (\ref{dcj}), we get 
\begin{equation*}
\begin{aligned}
    C^j(\alpha, 0) - C^j(\alpha, \bt) 
    &\le N V(\alpha) - N e^\alpha w(\alpha) \norm{\bt} \gamma_1^o + n w(\alpha) \norm{\bt} C \\
    &=NV(\alpha) - w(\alpha) (N e^\alpha \gamma_1^o - n C) \norm{\bt}.
\end{aligned}
\end{equation*}
At the minimizer $(\alpha^j_N,\beta_N)$, this becomes
$$
0\le   C^j(\alpha_N^j, 0) - C^j(\alpha_N^j, \bt_N) 
\le NV(\alpha_N^j) - w(\alpha_N^j) (N e^{\alpha_N^j} \gamma_1^o - n C) \norm{\bt_N},
$$
which implies
\begin{equation}
\norm{\bt_N}\left(\gamma_1^o - \frac{nC}{N}e^{-\alpha_N^j}\right) \le \frac{V(\alpha_N^j)}{e^{\alpha_N^j}w(\alpha_N^j)}.
\label{vsup}
\end{equation}
The right side is bounded almost surely for large $\alpha_N^j$, by the right-tail condition
in Definition~\ref{d:wti}.
Through any subsequence $N_m$ through which $\|\beta_{N_m}\|$ grows without bound, this inequality is eventually violated unless $\alpha^{j}_{N_m}\to-\infty$.
Thus, if $\|\beta_{N_m}\|$ is unbounded, we must have $\alpha^{j}_{N_m}\equiv \alpha_{N_m} + \beta_{N_m}^{\top}x_j
\to -\infty$. Suppose $\norm{\beta_{N_m}}$ remains bounded. We know from (\ref{minlim}) that $\min_i\{\alpha_{N_m} + \beta^{\top}_{N_m}x_i\}\to -\infty$, so we must have $\alpha_{N_m}\to -\infty$,
and thus we again have $\alpha_{N_m} + \beta_{N_m}^{\top}x_j\to -\infty$. We conclude that
$\alpha_N + \beta_N^{\top}x_j\to -\infty$, for all $j=1,\dots,n$, and thus
$\alpha_N + \beta_N^{\top}\bar{x}\to -\infty$ a.s.
\end{proof}

We complete the proof of Proposition~\ref{p:abi} by showing that $\|\beta_N\|$
remains bounded. In light of Lemma~\ref{l:step1i}, boundedness of $\|\beta_N\|$
implies that $\alpha_N\to -\infty$ a.s., as required for (\ref{ablimsi}).

\begin{lemma}[Step 2]\label{l:step2i}
Under the conditions of Proposition~\ref{p:abi}, $\limsup_{N \to \infty} \norm{\bt_N} \le 1/\gamma$ where $\gamma = (1-\lambda) \epsilon\delta$, with $\epsilon,\delta>0$ the surrounding parameters in Condition \ref{assump_surr_mean},
and $\lambda$ the exponential parameter in Definition~\ref{d:wti}. 
\end{lemma}

\begin{proof}[Proof of Lemma~\ref{l:step2i}]
We will work with the centered loss, centered around $\bar{x}$,
\begin{equation}\label{centerCost}
\tilde{C}_N(\alpha,\bt) = \sum_{i=1}^{n} -U(\alpha + \bt^\top( x_i - \bar{x})) 
+ \sum_{i=1}^N V(\alpha + \bt^{\top}(X_i-\bar{x})),
\end{equation}
which is minimized at $(\tilde{\alpha}_N, \tilde{\bt}_N)$, with
\[
\tilde{\alpha}_N = \alpha_N + \bt_N^\top \bar{x}, \quad \tilde{\bt}_N = \bt_N.
\]
For any $\alpha\in\R$ and $\bt \in \R^d$,
\begin{equation}\label{dct}
	\begin{aligned}
		&\tilde{C}_N(\alpha, 0) - \tilde{C}_N(\alpha, \bt) \\
		&= \sum_{i=1}^n [-U(\alpha) + U(\alpha+\bt^\top (x_i-\bar{x}))] + N V(\alpha) 
		- \sum_{i=1}^N V(\alpha + \bt^{\top}(X_i-\bar{x})).
	\end{aligned}
\end{equation}
We will bound the expression on the right.

Using the concavity of $U$, as in (\ref{ubnd}), we get
\begin{equation}\label{u_bound}
     \sum_{i=1}^n [-U(\alpha) + U(\alpha+\bt^\top (x_i-\bar{x}))] \le \sum_{i=1}^n w(\alpha) \bt^\top (x_i-\bar{x}) = 0.
\end{equation}
Using the convexity of $V$, as in (\ref{vbnd}), we get
\begin{equation}\label{vstep}
\sum_{i=1}^N V(\alpha + \bt^\top (X_i-\bar{x})) \ge
e^\alpha w(\alpha)\sum_{i=1}^N  [\bt^\top (X_i-\bar{x})]_+.
\end{equation}
Recalling that $\epsilon$ and $\delta$ are the surrounding parameters in Condition~\ref{assump_surr_mean},
let $\epsilon_1 \in (0,\epsilon)$ and $\delta_1 \in (0,\delta)$. By Lemma~\ref{l:surround}, the empirical distribution of $X_1,\dots,X_N$
surrounds $\bar{x}$ with parameters $(\epsilon_1,\delta_1)$ for all sufficiently large $N$, a.s., and so
\[
\inf_{\omega \in \Omega} \frac{1}{N}\sum_{i=1}^N [(X_i-\bar{x})^\top \omega]_+ 
\ge
\inf_{\omega \in \Omega} \frac{1}{N}\sum_{i:(X_i-\bar{x})^\top \omega>\epsilon_1} [(X_i-\bar{x})^\top \omega]_+ \ge \epsilon_1\delta_1 \equiv \gamma_1>0.
\]
Applying this bound in (\ref{vstep}), we get, for sufficiently large $N$, a.s.,
\begin{equation}\label{v_bound}
\sum_{i=1}^N V(\alpha + \bt^\top (X_i-\bar{x}))  \ge 
N e^\alpha w(\alpha) \norm{\bt} \gamma_1.
\end{equation}
Now, by applying (\ref{u_bound}) and (\ref{v_bound}) in (\ref{dct}) we get
\begin{equation}\label{dct2}
\tilde{C}_N(\alpha, 0) - \tilde{C}_N(\alpha, \bt) \le
NV(\alpha) - N e^\alpha w(\alpha) \norm{\bt} \gamma_1.
\end{equation}
Let $C = \lim_{u \to -\infty} w(u)/(e^{-\lambda u}h(u))$, with $\lambda$ and $h$ as in Definition~\ref{d:wti}. 
Take any $\epsilon_h \in (0, (1-\lambda)/C)$, 
By property (iii) in Definition~\ref{d:wti}, $\liminf_{u \to -\infty} h'(u)/h(u) \ge 0$, so there exists some $u_h < 0$ such that for any $u \le u_h$,
\begin{equation}\label{h-ord}
    -\epsilon_h h(u) \le h'(u).
\end{equation}
For any $\epsilon_0>0$,
there is a $u_0 \le u_h$ such that for all $u \le u_0$,
\[
(1-\epsilon_0) C e^{-\lambda u} h(u) \le w(u) \le (1+\epsilon_0) C e^{-\lambda u} h(u).
\]
These bounds yield, for $u\le u_0 \le u_h$,
\begin{equation*}
    \begin{aligned}
V(u) &\le (1+\epsilon_0) C \int_{-\infty}^u e^{(1-\lambda) s} h(s) ds 
 \\ 
 &= (1+\epsilon_0) C \left ( \frac{e^{(1-\lambda) u}}{1-\lambda} h(u) - \int_{-\infty}^u \frac{e^{(1-\lambda) s}}{1-\lambda} h'(s) ds \right )
 \\
 &\le  (1+\epsilon_0) C \left ( \frac{e^{(1-\lambda) u}}{1-\lambda} h(u) + \int_{-\infty}^u \frac{e^{(1-\lambda) s}}{1-\lambda} \epsilon_h h(s) ds \right ) 
 \\
 &= (1+\epsilon_0) C \frac{e^{(1-\lambda) u}}{1-\lambda} h(u) + C \frac{\epsilon_h}{1-\lambda} (1+\epsilon_0) \int_{-\infty}^u e^{(1-\lambda) s} h(s) ds \\
 &\le (1+\epsilon_0) C \frac{e^{(1-\lambda) u}}{1-\lambda} h(u) + C \frac{1+\epsilon_0}{1-\epsilon_0} \frac{\epsilon_h}{1-\lambda} V(u),
\end{aligned}
\end{equation*}
where going from the first line to the second we used integration by parts and from the second to 
the third we applied (\ref{h-ord}). 

This bound holds, in particular, for any 
$\epsilon_0 \in (0, \frac{1-\lambda - C\epsilon_h}{1-\lambda + C\epsilon_h})$, 
recalling that we took $C\epsilon_h<1-\lambda$. For any such $\epsilon_0$,
we have $C \frac{1+\epsilon_0}{1-\epsilon_0} \frac{\epsilon_h}{1-\lambda} < 1$, and so
\begin{equation}
V(u) \le \frac{(1+\epsilon_0)C e^{(1-\lambda) u} h(u)/(1-\lambda)}{1- C\frac{1+\epsilon_0}{1-\epsilon_0} \frac{\epsilon_h}{1-\lambda}} = \frac{(1+\epsilon_0)C e^{(1-\lambda) u} h(u)}{1- \lambda - C\epsilon_h (1+\epsilon_0)/(1-\epsilon_0)}.
\label{vepbnd}
\end{equation}
By the optimality of $\tilde{\alpha}_N$ and $\tilde{\beta}_N$, and noting that $\beta_N = \tilde{\beta}_N$, 
\begin{equation}\label{eqn-opt}
    \tilde{C}_N(\tilde{\alpha}_N, 0) - \tilde{C}_N(\tilde{\alpha}_N, \tilde{\beta}_N) 
= \tilde{C}_N(\tilde{\alpha}_N, 0) - \tilde{C}_N(\tilde{\alpha}_N, \beta_N)
\ge 0.
\end{equation}
From Lemma~\ref{l:step1i}, we have $\tilde{\alpha}_N \to -\infty$, so we may take $N$ large enough that  $\tilde{\alpha}_N \le u_0$, and then (\ref{eqn-opt}),
(\ref{vepbnd}), and (\ref{dct2}) yield
\begin{equation*}
\begin{aligned}
0\le \tilde{C}_N(\tilde{\alpha}_N, 0) - \tilde{C}_N(\tilde{\alpha}_N, \bt_N) 
	&\le N V(\tilde{\alpha}_N) - N  e^{\tilde{\alpha}_N} w(\tilde{\alpha}_N) \norm{\bt_N} \gamma_1 \\
	&\le N C \frac{(1+\epsilon_0)e^{(1-\lambda)\tilde{\alpha}_N} h(\tilde{\alpha}_N)}{1-\lambda - C\epsilon_h(1+\epsilon_0)/(1-\epsilon_0)} - N C (1-\epsilon_0) e^{(1-\lambda) \tilde{\alpha}_N} h(\tilde{\alpha}_N) \norm{\bt_N} \gamma_1 \\
	&= N C  e^{(1-\lambda)\tilde{\alpha}_N} h(\tilde{\alpha}_N) \left (\frac{1+\epsilon_0}{1-\lambda-C\epsilon_h(1+\epsilon_0)/(1-\epsilon_0)} - (1-\epsilon_0) \norm{\bt_N} \gamma_1 \right ).
\end{aligned}
\end{equation*}
Therefore,  we have
\[
\norm{\bt_N}\le \frac{1+\epsilon_0}{(1-\lambda)(1-\epsilon_0)-C(1+\epsilon_0)\epsilon_h} \frac{1}{\gamma_1}
\]
for all sufficiently large $N$, a.s.  

Notice that $\epsilon_h$ and $\epsilon_0$ may be taken arbitrarily close to $0$, and
by Lemma~\ref{l:surround} $\epsilon_1$ and $\delta_1$ may be
taken arbitrarily close to $\epsilon$ and $\delta$, respectively, so 
recalling that $\gamma_1=\epsilon_1\delta_1$,
we conclude that
\[
\limsup_N \|\beta_N\| \le \frac{1}{(1-\lambda) \epsilon\delta} = \frac{1}{\gamma}, \quad \mbox{a.s.}
\]
\end{proof}

We can say more about the rate at which $\alpha_N$ diverges:

\begin{corollary}\label{c:alog}
Under the conditions of Proposition~\ref{p:abi}, there is a constant $\kappa$ for
which 
\begin{equation}
\limsup_{N\to\infty}(\alpha_N + \log N) < \kappa, \quad \mbox{ a.s.}
\label{kapbnd}
\end{equation}
\end{corollary}

\begin{proof}[Proof of Corollary~\ref{c:alog}]
Let $K = \max_{i=1,2,\dots,n} \|x_i\|$ and let $r > 1/\gamma$. For $N$ sufficiently large,
\begin{equation}\label{bound-bx}
    \max_{i=1,\dots,n} |\beta_N^{\top} x_i|\le \max_{i=1,\dots,n} \|\beta^{\top}_N\|\|x_i\| \le r K, \quad \text{a.s.}
\end{equation}
Because Lemma~\ref{lemma:min_div} holds for any weight function (bounded or
unbounded) satisfying Definition~\ref{d:wti}, 
(\ref{eqn:poly_rate}) implies that for $N$ sufficiently large,
\[
N e^{\alpha_N + \beta_N^\top x_j(N)} \delta_1^o \le n, \quad \mbox{a.s.}
\]
Taking logs on both sides and rearranging yields 
\[
\alpha_N + \log N \le \log(n/\delta_1^o) - \beta_N^\top x_j(N), \quad \mbox{a.s.}
\]
Therefore, for any $\kappa > \log(n/\delta_1^o) + r K$ we can apply (\ref{bound-bx}) to get (\ref{kapbnd}).
\end{proof}

\subsection{Proof of Theorem~\ref{bdd-thm}}

We will need the following lemma in the proof of Theorem~\ref{bdd-thm}.
\begin{lemma}\label{lemma:min-u0}
Suppose the conditions in Corollary~\ref{c:alog} hold. Then for any $u_0 \in \R$
\begin{equation}
\max_{i=1,\dots,N_n}\{ \alpha_N + \beta^{\top}_N X_i\} \le u_0,
\label{maxlog}
\end{equation}
for all sufficiently large $N$, a.s.
\end{lemma}
\begin{proof}[Proof of Lemma~\ref{lemma:min-u0}]
Condition~\ref{assump_tail} implies that for any $\epsilon'>0$,
$$
\sum_{k=1}^{\infty} \PR(e^{r\|X\|}>\epsilon' k) <\infty;
$$
so for any $v \in\R$,
$$
\sum_{k=1}^{\infty} \PR(r\|X\|> v + \log k) <\infty.
$$
By Theorem 3.5.1 of Embrechts, Kluppelberg, and Mikosch \cite{embook},
this implies
$$
\PR(\max_{i=1,\dots,k}r\|X_i\| > v + \log k\; \mbox{i.o.})=0,
$$
where ``i.o.''\ means infinitely often.
It follows that
$$
\PR(\max_{i=1,\dots,N_n}r\|X_i\| > v + \log N_n\; \mbox{i.o.})=0.
$$
Now $|\beta^{\top}_N X_i| \le \|\beta_N\|\cdot \|X_i\|
\le r\|X_i\|$,
for all sufficiently large $N$, a.s., so
$$
\PR(\max_{i=1,\dots,N}|\beta^{\top}_N X_i| > v + \log N\; \mbox{i.o.})=0.
$$
This means that, a.s., for all sufficiently large $N$,
$$
\max_{i=1,\dots,N_n}\beta^{\top}_N X_i \le v + \log N_n.
$$
By Corollary~\ref{c:alog}, $\log N < \kappa - \alpha_N$, for all sufficiently large $N$, a.s., so
$$
\max_{i=1,\dots,N}\{\alpha_N + \beta^{\top}_N X_i\} \le v + \kappa.
$$
Choosing $v = u_0-\kappa$ proves (\ref{maxlog}).
\end{proof}

\begin{proof}[Proof of Theorem~\ref{bdd-thm}] 
Taking the partial derivative of $\bar{C}_N$ with respect to $\alpha$ and 
its gradient with respect to $\beta$,
we get
\[
\frac{\partial \bar{C}_N}{\partial \alpha} 
= \sum_{i=1}^n -w(\si) + 
\sum_{i=1}^N e^{\alpha +\beta^{\top}X_i} w(\alpha +\beta^{\top}X_i)
\]
\[
\frac{\partial \bar{C}_N}{\partial \bt} = \sum_{i=1}^n -w(\si) x_i + 
\sum_{i=1}^N e^{\alpha +\beta^{\top}X_i} w(\alpha +\beta^{\top}X_i)X_i.
\]
At the minimizer $(\alpha_N,\beta_N)$, these derivatives equal zero, so, a.s.,
for all sufficiently large $N$, 
\begin{equation}
\sum_{i=1}^n w(\alpha_N + \beta_N^{\top}x_i) = 
\sum_{i=1}^N e^{\alpha_N +\beta_N^{\top}X_i} w(\alpha_N +\beta_N^{\top}X_i)
\label{foc1}
\end{equation}
and
\begin{equation}
\sum_{i=1}^n w(\alpha_N + \beta_N^{\top}x_i)x_i 
= 
\sum_{i=1}^N e^{\alpha_N +\beta^{\top}_NX_i} w(\alpha_N +\beta_N^{\top}X_i)X_i.
\label{foc2}
\end{equation}
Because $\alpha_N\to-\infty$ and $\beta_N$ is bounded a.s., for any $x\in\R^d$,
\begin{equation}
\frac{w(\alpha_N+\beta_N^{\top} x)}{w(\alpha_N)} - e^{-\lambda \beta_N^\top x} \to 0 \quad \mbox{a.s.}
\label{ww1}
\end{equation}
Therefore
$$
\frac{\sum_{i=1}^n w(\alpha_N + \beta_N^{\top}x_i)x_i}{\sum_{i=1}^n w(\alpha_N + \beta_N^{\top}x_i)} - \frac{\sum_{i=1}^n x_i e^{-\lambda \beta_N^\top x_i}}{\sum_{i=1}^n e^{-\lambda \beta_N^\top x_i}} 
= 
\frac{\sum_{i=1}^n w(\alpha_N + \beta_N^{\top}x_i)x_i/w(\alpha_N)}{\sum_{i=1}^n w(\alpha_N + \beta_N^{\top}x_i)/w(\alpha_N)} - 
\frac{\sum_{i=1}^n x_i e^{-\lambda \beta_N^\top x_i}}{\sum_{i=1}^n e^{-\lambda \beta_N^\top x_i}}
\to 0 \quad \mbox{a.s.}
$$
Taking the ratios of the two sides in (\ref{foc1})--(\ref{foc2}), we get
\begin{equation}
\frac{\sum_{i=1}^N e^{\beta_N^{\top}X_i} w(\alpha_N +\beta_N^{\top}X_i)X_i}
{\sum_{i=1}^N e^{\beta_N^{\top}X_i} w(\alpha_N +\beta_N^{\top}X_i)}
- \frac{\sum_{i=1}^n x_i e^{-\lambda \beta_N^\top x_i}}{\sum_{i=1}^n e^{-\lambda \beta_N^\top x_i}}
\to 0 \quad \mbox{a.s.}
\label{intlim}
\end{equation} 

\begin{lemma}\label{l:zone}
Suppose the conditions of Theorem~\ref{bdd-thm} hold. 
Suppose there is a possibly random $\beta \in \R^d$ 
independent of $\{X_i\}_{i \in \mathbb N}$ and 
a possibly stochastic subsequence
$N_n\to\infty$ through which $\beta_{N_n}\to\beta$, a.s. Then
$$
\frac{1}{N_n}\sum_{i=1}^{N_n} e^{\beta_{N_n}^{\top}X_i} 
\frac{w(\alpha_{N_n} +\beta_{N_n}^{\top}X_i)}{w(\alpha_{N_n})}
\to
\int e^{(1-\lambda) \beta^{\top}x}\, dF_0(x),\quad \mbox{a.s.},
$$
and
$$
\frac{1}{N_n}\sum_{i=1}^{N_n} e^{\beta_{N_n}^{\top}X_i} 
\frac{w(\alpha_{N_n} +\beta_{N_n}^{\top}X_i)}{w(\alpha_{N_n})}X_i
\to
\int e^{(1-\lambda) \beta^{\top}x}x\, dF_0(x),\quad \mbox{a.s.}
$$
\end{lemma}

\begin{proof}
Condition~\ref{assump_tail} ensures that
the limiting integrals are well-defined and finite because Lemma~\ref{l:step2i}
implies that $\|\beta^{\top}x\| \le \|\beta\|\|x\|\le \|x\|/\gamma < r\|x\|$.
We detail the argument for the second limit; the first limit works the same way.
Write
\begin{eqnarray}
\lefteqn{
\left\|\frac{1}{N_n}\sum_{i=1}^{N_n} e^{\beta_{N_n}^{\top}X_i} 
\frac{w(\alpha_{N_n} +\beta_{N_n}^{\top}X_i)}{w(\alpha{N_n})}X_i
-\int e^{(1-\lambda) \beta^{\top}x}x\, dF_0(x)\right\|} && \nonumber \\
&\le&
\left\|\frac{1}{N_n}\sum_{i=1}^{N_n} e^{\beta_{N_n}^{\top}X_i} \frac{w(\alpha_{N_n} +\beta_{N_n}^{\top}X_i)}{w(\alpha_{N_n})}X_i
-
\frac{1}{N_n}\sum_{i=1}^{N_n} e^{(1-\lambda) \beta_{N_n}^{\top}X_i} X_i\right\|
\label{d1} \\
&& +
\left\|\frac{1}{N_n}\sum_{i=1}^{N_n} e^{(1-\lambda) \beta_{N_n}^{\top}X_i} X_i
- \frac{1}{N_n}\sum_{i=1}^{N_n} e^{(1-\lambda) \beta^{\top}X_i} X_i\right\|
\label{d2} \\
&& +
\left\|\frac{1}{N_n}\sum_{i=1}^{N_n} e^{(1-\lambda)\beta^{\top}X_i} X_i
-\int e^{(1-\lambda)\beta^{\top}x}x\, dF_0(x)\right\|.
\label{d3}
\end{eqnarray}
We examine these terms in reverse order. The term in (\ref{d3}) vanishes, a.s.,
by the strong law of large numbers.

Turning next to (\ref{d2}), we use Taylor's theorem to write
$$
e^{(1-\lambda)\beta_{N_n}^{\top}X_i} = 
e^{(1-\lambda)\beta^{\top}X_i} + e^{(1-\lambda)\tilde{\beta}_{N_n,i}^{\top}X_i}(1-\lambda)(\beta_{N_n}-\beta)^{\top}X_i,
$$
for some $\tilde{\beta}_{N_n,i}$ on the line segment connecting
$\beta_{N_n}$ and $\beta$.
So,
\begin{eqnarray}
&\left\|\frac{1}{N_n}\sum_{i=1}^{N_n} e^{(1-\lambda)\beta_{N_n}^{\top}X_i}X_i
-\frac{1}{N_n}\sum_{i=1}^{N_n} e^{(1-\lambda)\beta^{\top}X_i}X_i\right\| \\
&=
\left\|
\frac{1}{N_n}\sum_{i=1}^{N_n}  e^{(1-\lambda)\tilde{\beta}_{N_n,i}^{\top}X_i}(1-\lambda)(\beta_{N_n}-\beta)^{\top}X_i\cdot X_i
\right\| \nonumber \\
&\le
(1-\lambda) \|\beta_{N_n}-\beta\| \frac{1}{N_n}\sum_{i=1}^{N_n} e^{(1-\lambda) \|\tilde{\beta}_{N_n,i}\|\cdot\|X_i\|}\|X_i\|^2.
\label{diffs3}
\end{eqnarray}
With $r$ as in Condition~\ref{assump_tail} and any $r>r'>1/\gamma$,
we know from Lemma~\ref{l:step2i} that, a.s.,
$\|\beta_{N_n}\|\le r'$ for all sufficiently large $N_n$, and
$\|\beta\|\le 1/\gamma$, so
$\|\tilde{\beta}_{N_n,i}\|\le r'$, $i=1,\dots,N_n$, a.s., for all
sufficiently large $N_n$ because each
$\tilde{\beta}_{N_n,i}$ is a convex combination of $\beta_{N_n}$ and $\beta$.
Thus,
$$
\limsup_{N_n\to\infty}
\frac{1}{N_n}\sum_{i=1}^{N_n} e^{(1-\lambda)\|\tilde{\beta}_{N_n,i}\|\cdot\|X_i\|}\|X_i\|^2
\le
\lim_{N_n\to\infty}
\frac{1}{N_n}\sum_{i=1}^{N_n} e^{(1-\lambda) r'\|X_i\|}\|X_i\|^2 < \infty,
$$
in light of Condition~\ref{assump_tail}. Thus, (\ref{diffs3}) vanishes as $N_n\to\infty$, a.s., 
and then (\ref{d2}) vanishes, a.s., as $N_n\to\infty$.

To bound (\ref{d1}), we need to bound 
\[
\left | \frac{w(\alpha_{N_n} +\beta_{N_n}^{\top}X_i)}{w(\alpha_{N_n})} - e^{-\lambda \beta_{N_n}^\top X_i} \right |.
\]
By Definition~\ref{d:wti}, for any $\epsilon \in (0,1)$ we may choose $u_0 < 0$ such that for any $u \le u_0$ and $u+s \le u_0$,
\[
(1-\epsilon) C e^{-\lambda u} h(u) \le w(u) \le (1+\epsilon) C e^{-\lambda u} h(u)
\]
and 
\[
(1-\epsilon) C e^{-\lambda (u+s)} h(u+s) \le w(u+s) \le (1+\epsilon) C e^{-\lambda (u+s)} h(u+s).
\]
We therefore have
\begin{equation}\label{exp:w-ratio}
    \frac{1-\epsilon}{1+\epsilon} e^{-\lambda s} \frac{h(u+s)}{h(u)} \le \frac{w(u+s)}{w(u)} \le \frac{1+\epsilon}{1-\epsilon} e^{-\lambda s} \frac{h(u+s)}{h(u)}.
\end{equation}
By condition (\ref{h-rate}) in Definition~\ref{d:wti}, we can find some $C_1 > 0$ and some $u_1 \le u_0 < 0$  such that for any $u \le u_1$ and $u+s \le u_1$, 
\[
\left | \frac{h(u+s)}{h(u)} - 1 \right | \le \frac{2\epsilon}{1+\epsilon} \max\{C_1, e^{\xi|s|}\}.
\]
Let $g(s) = \max\{C_1, e^{\xi|s|}\}$. Expanding the absolute value yields
\[
1 - \frac{2\epsilon}{1+\epsilon} g(s)
\le \frac{h(u+s)}{h(u)} \le 1 + \frac{2\epsilon}{1+\epsilon} g(s).
\]
Because $g(s) > 0$, we have
\[
1 - \frac{2\epsilon}{1-\epsilon}g(s)
\le 1 - \frac{2\epsilon}{1+\epsilon} g(s)
\le \frac{h(u+s)}{h(u)} \le 1 + \frac{2\epsilon}{1+\epsilon} g(s).
\]
Substituting back into (\ref{exp:w-ratio}), we get
\[
\frac{1-\epsilon}{1+\epsilon} e^{-\lambda s} \left(1 - \frac{2\epsilon}{1-\epsilon} g(s)\right) \le \frac{w(u+s)}{w(u)}
\le \frac{1+\epsilon}{1-\epsilon} e^{-\lambda s}\left(1 + \frac{2\epsilon}{1+\epsilon} g(s)\right).
\]
Subtracting $e^{-\lambda s}$ from all three parts of the above inequality, we have
\[
 e^{-\lambda s} \left (\frac{1-\epsilon}{1+\epsilon} - 1 \right)- e^{-\lambda s}
\frac{2\epsilon}{1+\epsilon} g(s) \le \frac{w(u+s)}{w(u)} - e^{-\lambda s}
\le e^{-\lambda s} \left ( \frac{1+\epsilon}{1-\epsilon} - 1\right) + e^{-\lambda s} \frac{2\epsilon}{1-\epsilon} g(s),
\]
and so
\[
-\frac{2\epsilon}{1+\epsilon} e^{-\lambda s} (1 + g(s)) \le \frac{w(u+s)}{w(u)} - e^{-\lambda s} 
\le \frac{2\epsilon}{1-\epsilon} e^{-\lambda s} (1 + g(s)).
\]
Because 
\[
-\frac{2\epsilon}{1-\epsilon} e^{-\lambda s} (1 + g(s)) \le -\frac{2\epsilon}{1+\epsilon} e^{-\lambda s} (1 + g(s)),
\]
letting $\epsilon_0 = 2\epsilon/(1-\epsilon)$ we get
\[
\left |\frac{w(u+s)}{w(u)} - e^{-\lambda s}\right | \le \epsilon_0 e^{-\lambda s} (1 + g(s)).
\]
This implies, for all sufficiently large $N_n$,
\[
\left | \frac{w(\alpha_{N_n} +\beta_{N_n}^{\top}X_i)}{w(\alpha_{N_n})} - e^{-\lambda \beta_{N_n}^\top X_i} \right |
\le 
\epsilon_0 e^{-\lambda \beta_{N_n}^\top X_i} (1 + g( |\beta_{N_n}^\top X_i|)).
\]
Now
\begin{equation*}
    \begin{aligned}
    &\left\|\frac{1}{N_n}\sum_{i=1}^{N_n} e^{\beta_{N_n}^{\top}X_i} \frac{w(\alpha_{N_n} +\beta_{N_n}^{\top}X_i)}{w(\alpha_{N_n})}X_i
-
\frac{1}{N_n}\sum_{i=1}^{N_n} e^{(1-\lambda) \beta_{N_n}^{\top}X_i} X_i\right\| \\
&=\left\|\frac{1}{N_n}\sum_{i=1}^{N_n} e^{\beta_{N_n}^{\top}X_i} \left ( \frac{w(\alpha_{N_n} +\beta_{N_n}^{\top}X_i)}{w(\alpha_{N_n})} - e^{-\lambda \beta_{N_n}^\top X_i} \right ) X_i\right \| \\
&\le \frac{1}{N_n}\sum_{i=1}^{N_n} e^{\beta_{N_n}^{\top}X_i} \left | \frac{w(\alpha_{N_n} +\beta_{N_n}^{\top}X_i)}{w(\alpha_{N_n})} - e^{-\lambda \beta_{N_n}^\top X_i} \right | \|X_i\| \\
&\le \epsilon_0 \frac{1}{N_n}\sum_{i=1}^{N_n} e^{(1-\lambda) \beta_{N_n}^\top X_i} (1 + g(|\beta_{N_n}^\top X_i|))  \norm{X_i} \\
&\le \epsilon_0 (1+C_1) \frac{1}{N_n}\sum_{i=1}^{N_n} e^{(1-\lambda) \beta_{N_n}^\top X_i} \norm{X_i} + \epsilon_0 \frac{1}{N_n}\sum_{i=1}^{N_n} e^{(1-\lambda) \beta_{N_n}^\top X_i} e^{\xi |\beta_{N_n}^\top X_i|} \norm{X_i},
    \end{aligned}
\end{equation*}
where the last inequality uses $g(u) = \max\{C_1, e^{\xi |u|} \} \le C_1 + e^{\xi |u|}$. For any $r'$ such that 
$\max\{1, 1-\lambda+\xi\}/\gamma < r' < r$,
$$
\limsup_{N_n\to\infty}
\frac{1}{N_n}\sum_{i=1}^{N_n} e^{(1-\lambda) \beta_{N_n}^{\top}X_i} \norm{X_i} 
\le
\limsup_{N_n\to\infty}
\frac{1}{N_n}\sum_{i=1}^{N_n} e^{(1-\lambda) r'\|X_i\|} \|X_i\| < \infty
$$
and
\begin{equation*}
    \limsup_{N_n\to\infty}
\frac{1}{N_n}\sum_{i=1}^{N_n} e^{(1-\lambda) \beta_{N_n}^\top X_i} e^{\xi |\beta_{N_n}^\top X_i|} \norm{X_i} \le \frac{1}{N_n}\sum_{i=1}^{N_n} e^{(1-\lambda+\xi)|\beta_{N_n}^\top X_i|} \norm{X_i} 
\le \frac{1}{N_n}\sum_{i=1}^{N_n} e^{r' \|X_i\|} \norm{X_i} < \infty.
\end{equation*}
As $\epsilon_0>0$ can be arbitrarily small, we have shown that
(\ref{d1}) vanishes, a.s., as do (\ref{d2}) and (\ref{d3}).
We have thus proved the second limit in the lemma. The first limit
works the same way.
\end{proof}

We can now complete the proof of Theorem~\ref{bdd-thm}.
Our bound on  $\norm{\beta_N}$ ensures that $\beta_N$ has at least one limit point; for any limit point $\tilde{\beta}$ of $\beta_N$, there is a subsequence $\beta_{N_n}$ such that $\beta_{N_n} \to \tilde{\beta}$ a.s. 
By combining the two limits in Lemma~\ref{l:zone} in (\ref{intlim}) we have
$$
\frac{\int e^{(1-\lambda)\tilde{\beta}^{\top}x}x\, dF_0(x)}{\int e^{(1-\lambda)\tilde{\beta}^{\top}x}\, dF_0(x)}
=\frac{\sum_{i=1}^n x_i e^{-\lambda \tilde{\beta}^\top x_i}}{\sum_{i=1}^n e^{-\lambda \tilde{\beta}^\top x_i}}, \quad \mbox{a.s.}
$$
We claim that there can be at most one $\beta_*\in\R^d$ satisfying this equation, and so $\tilde{\beta} = \beta_*$ a.s.
To see this, define  the cumulant generating function
$\psi(\bt) = \log \E[e^{\bt^\top W}]$ of $W=(1-\lambda) X_0 - \lambda X_1$, where $X_0 \sim F_0$, $X_1$ is uniform over of $x_1,\dots,x_n$, and $X_0$, $X_1$ are independent. 
Equation (\ref{bstar:exp}) reads $\nabla\psi(\beta_*) = 0$.
The surrounding condition on $F_0$ ensures that the support of $F_0$ has full dimension.
By Theorem 1.13(iv) of Brown \cite{brown}, this implies that $\psi$ is strictly convex.
Strict convexity implies that for $\beta\not=\beta_*$,
$$
\nabla\psi(\beta)\cdot (\beta_*-\beta)
< \psi(\beta_*)-\psi(\beta) <
\nabla\psi(\beta_*)\cdot (\beta_*-\beta),
$$
and thus $\nabla\psi(\beta)\not=\nabla\psi(\beta_*)$.

We have thus shown that any $\beta_{N_n}$ has an almost sure constant limit $\beta_*$, and we conclude that $\beta_N \to \beta_*$ a.s. 
\end{proof}

\subsection{Boundedness of \texorpdfstring{$V(u)/e^uw(u)$}{Lg}}
\label{s:vb}

Definition~\ref{d:wti} requires an upper bound on $V(u)/e^uw(u)$ for unbounded weight functions, and this condition is used in (\ref{vsup}). 
The following lemma shows that this condition is satisfied by a broad family
of weight functions, including $w(u) = Ce^{-\lambda u}$, $\lambda\in(0,1)$,
and $w(u) = Cu^{-k}$, $k>0$, for large $u$.

\begin{lemma}
Suppose there is an increasing log-convex function $g$ for which
$Cg(u)\le e^uw(u) \le g(u)$, for all $u\ge u_0$, for some $u_0\in \R$ with $g'(u_0)\not=0$,
and some $C>0$. Then
$V(u)/e^uw(u)$ is bounded above on $[u_0,\infty)$.
\end{lemma}

\begin{proof}
Theorem 2.1 of Gill, Pearce, and Pe\v{c}ari\'{c} \cite{gill} provides an upper bound on 
the integral of a log-convex function $g$, which yields, for any $u>u_0$,
$$
V(u) = V(u_0) + \int_{u_0}^u e^sw(s)\, ds \le V(u_0) + \int_{u_0}^u g(s)\,ds
\le V(u_0) + \frac{(u-u_0)(g(u)-g(u_0))}{\log g(u)-\log g(u_0)}.
$$
Since $g(u_0) \ge e^{u_0} w(u_0) \ge 0$, the bound remains valid if we remove $g(u_0)$ from the numerator of the last
term. Convexity of $\log g$ implies that $\log g(u)-\log g(u_0) \ge (\log g(u_0))' (u-u_0)$,
and $(\log g(u_0))' = g'(u_0)/g(u_0)$, so
$$
\frac{V(u)}{e^uw(u)} \le \frac{V(u_0)}{e^uw(u)} + \frac{g(u)}{e^uw(u)}
\frac{(u-u_0)}{\log g(u)-\log g(u_0)}
\le  \frac{V(u_0)}{e^{u_0}w(u_0)} + \frac{1}{C}\frac{g(u_0)}{g'(u_0)}.
$$
\end{proof}

\section{Proofs for Section 5}
\subsection{Proof of Proposition~\ref{p:dmix}}
\begin{proof}[Proof of Proposition~\ref{p:dmix}]
For any value $\mu$ of the common mean required by the constraint in (\ref{dmix}),
we know from Lemma~\ref{l:proj} that $D(G_i \|F_i)$, $i=0,1$, is minimized by
taking $G_i$ to be
an exponentially tilted distribution $F_{i,\beta_i}$, with $\nabla\psi_i(\beta_i)=\mu$,
$i=0,1$.
We then get
$$
D(G_i \|F_i) = \int \beta_i^{\top}x - \psi_i(\beta_i)\, dG_i = 
\beta_i^{\top}\mu - \psi_i(\beta_i), \quad i=0,1,
$$
and an objective function value of
\begin{equation}
\lambda[\beta_0^\top \mu - \psi_0(\beta_0)]
+
(1-\lambda)[\beta_1^\top \mu - \psi_1(\beta_1)].
\label{cmu}
\end{equation}
We can solve (\ref{dmix}) by minimizing (\ref{cmu}) over $\mu$, keeping in mind that
$\beta_0$ and $\beta_1$ depend on $\mu$.

Each function $\mu\mapsto \beta_i^\top \mu - \psi_i(\beta_i)$ in (\ref{cmu})
is the convex conjugate of $\psi_i$, $i=0,1$, at $\mu$, defined by
$$
\sup_b \{b^\top \mu - \psi_i(b)\},
$$
and is therefore convex in $\mu$. It follows that any point at which the
derivative of (\ref{cmu}) with respect to $\mu$ is zero minimizes (\ref{cmu}).

Differentiating  (\ref{cmu}) with respect to $\mu$, writing $\dot{\beta}_i$ for the 
derivative matrix of $\beta_i$ with respect to $\mu$, and setting the derivative equal
to zero to get
$$
\lambda[\dot{\beta_0}\cdot\mu + \beta_0 - \dot{\beta}_0\cdot\nabla\psi_0(\beta_0)]
+
(1-\lambda)[\dot{\beta_1}\cdot\mu + \beta_1 - \dot{\beta}_1\cdot\nabla\psi_1(\beta_1)]=0.
$$
But $\nabla\psi_i(\beta_i)=\mu$, $i=0,1$, so this equation simplifies to
$$
\lambda\beta_0 + (1-\lambda)\beta_1 = 0.
$$
The solution is then of the form
$\beta_0 = (1-\lambda)\beta$ and $\beta_1=-\lambda\beta$, where
$\beta$ solves
$$
\nabla\psi_0((1-\lambda)\beta) = \nabla\psi_1(-\lambda\beta),
$$
which is (\ref{psistar}).
\end{proof}

\section{Connection with Nonlinear Classifiers}\label{s:robust_nonlinear}
In this section, we provide further insight into the connection between (\ref{dmix}) and the original
classification problem by generalizing (\ref{eloss}) to the problem of minimizing
\begin{equation}
\bar{C}_{\lambda}(R)
=
\mathbb{E}\left[Y(1-\lambda)
e^{-\lambda R(X)} +
(1-Y)\lambda
e^{(1-\lambda)R(X)}\right]
\label{clam}
\end{equation}
over possibly nonlinear discriminant functions $R:\R^d\to\R$.
Suppose for simplicity that
$F_0$ and $F_1$ have densities $f_0$ and $f_1$. Then arguing as in 
Lemma 1 of Friedman, Hastie, and Tibshirani \cite{friedman},
(\ref{clam}) is minimized at
$$
R(x) = \log\frac{\pi_1}{\pi_0} + \log\frac{f_1(x)}{f_2(x)}.
$$
Making this substitution in (\ref{clam}) and simplifying yields
\begin{equation}
\bar{C}_{\lambda}(R)
=
\int (\pi_1f_1(x))^{1-\lambda}(\pi_0f_0(x))^{\lambda}\, dx.
\label{clam2}
\end{equation}
The case $\lambda=1/2$ appears in equation (28) of Eguchi and Copas \cite{eguchi}.
We can write (\ref{clam2}) as
\begin{equation}
\bar{C}_{\lambda}(R)
=
\pi_1^{1-\lambda}\pi_0^{\lambda}e^{(\lambda-1)D_{\lambda}(F_0 \|F_1)},
\label{clam3}
\end{equation}
using the R\'enyi divergence,
$$
D_{\lambda}(H \|F) = \frac{1}{\lambda-1}\log \int dH^{\lambda}dF^{(1-\lambda)}.
$$
The R\'enyi divergence has a representation in terms of the Kullback-Leibler divergence as
$$
(1-\lambda)D_{\lambda}(H \|F) = \inf_G \{\lambda D(G \|H) + (1-\lambda)D(G \|F)\},
$$
(Erven and Harremo\"{e}s \cite{erven}, Theorem 30)
which allows us to write (\ref{clam3}) as
\begin{equation}
\bar{C}_{\lambda}(R)
=
\pi_1^{1-\lambda}\pi_0^{\lambda}e^{-\inf_G \{\lambda D(G \|F_0) + (1-\lambda)D(G \|F_1)\}}.
\label{clam4}
\end{equation}
In other words, when we drop the requirement that the discriminant function be linear,
the minimal loss (\ref{clam}) is determined by
$$
\inf_G \{\lambda D(G \|F_0) + (1-\lambda)D(G \|F_1)\}.
$$
The loss in (\ref{clam4}) is small when there is no $G$ that is ``close'' to both $F_0$
and $F_1$ in the sense of this weighted divergence.
When we 
require the discriminant function  $R(\cdot)$ in (\ref{clam}) to be linear,
the optimal $\beta$ is
determined by (\ref{dmix}), which relaxes the constraint that $G_0=G_1$ to the requirement
that these two distributions have the same mean.
But this discussion shows that the effect of $\lambda$ discussed in
Section~\ref{s:robust2} extends at least in part to nonlinear discriminant
functions defined through (\ref{clam}).

\section{Supplementary Tables}
\subsection{Freddie Mac Dataset AUC}\label{appendix-fredmac}

Tables \ref{t:auc_logit}--\ref{t:auc_exp9} report AUC values for training, validation, and test sets
for four choices of classifiers. Results are indexed by training year, so test results for 2003,
for example, are based on defaults predicted for the first quarter of 2006. In all cases,
the test and validation AUCs are close to the training AUCs and above 0.8.

\begin{table}[ht]
\begin{minipage}[b]{0.45\linewidth}\centering
\begin{tabular}{r|rrr}
\hline
Year &  Train &  Val &   Test \\
\hline
2003 & 0.8899 & 0.8748 & 0.8867 \\
2004 & 0.8514 & 0.8368 & 0.8484 \\
2005 & 0.8456 & 0.8465 & 0.8457 \\
2006 & 0.8314 & 0.8220 & 0.8294 \\
2007 & 0.8258 & 0.8234 & 0.8254 \\
2008 & 0.8487 & 0.8512 & 0.8492 \\
2009 & 0.8763 & 0.8769 & 0.8764 \\
2010 & 0.8421 & 0.8663 & 0.8470 \\
2011 & 0.8730 & 0.8539 & 0.8683 \\
2012 & 0.6573 & 0.6379 & 0.6536 \\
2013 & 0.8589 & 0.8383 & 0.8546 \\
\hline
\end{tabular}
\caption{AUC, Logistic Regression}\label{t:auc_logit}
\end{minipage}
\hspace{0.5cm}
\begin{minipage}[b]{0.45\linewidth}\centering
\begin{tabular}{r|rrr}
\hline
Year &  Train &  Val &   Test \\
\hline
2003 & 0.8911 & 0.8766 & 0.8880 \\
2004 & 0.8530 & 0.8338 & 0.8490 \\
2005 & 0.8479 & 0.8493 & 0.8482 \\
2006 & 0.8322 & 0.8206 & 0.8298 \\
2007 & 0.8277 & 0.8259 & 0.8273 \\
2008 & 0.8503 & 0.8526 & 0.8508 \\
2009 & 0.8783 & 0.8796 & 0.8785 \\
2010 & 0.8440 & 0.8636 & 0.8480 \\
2011 & 0.8748 & 0.8580 & 0.8707 \\
2012 & 0.8759 & 0.8583 & 0.8726 \\
2013 & 0.8600 & 0.8372 & 0.8553 \\
\hline
\end{tabular}
\caption{AUC, $\lambda = 0.1$}\label{t:auc_exp1}
\end{minipage}\\

\begin{minipage}[b]{0.45\linewidth}\centering
\begin{tabular}{r|rrr}
\hline
Year &  Train &  Val &   Test \\
\hline
2003 & 0.8914 & 0.8771 & 0.8883 \\
2004 & 0.8555 & 0.8379 & 0.8519 \\
2005 & 0.8531 & 0.8549 & 0.8534 \\
2006 & 0.8339 & 0.8235 & 0.8318 \\
2007 & 0.8281 & 0.8264 & 0.8278 \\
2008 & 0.8514 & 0.8534 & 0.8518 \\
2009 & 0.8794 & 0.8793 & 0.8794 \\
2010 & 0.8444 & 0.8654 & 0.8487 \\
2011 & 0.8754 & 0.8626 & 0.8723 \\
2012 & 0.8760 & 0.8617 & 0.8733 \\
2013 & 0.8600 & 0.8367 & 0.8551 \\
\hline
\end{tabular}
\caption{AUC, $\lambda = 0.5$}\label{t:auc_exp5}
\end{minipage}
\hspace{0.5cm}
\begin{minipage}[b]{0.45\linewidth}\centering
\begin{tabular}{r|rrr}
\hline
Year &  Train &  Val &   Test \\
\hline
2003 & 0.8875 & 0.8740 & 0.8847 \\
2004 & 0.8537 & 0.8376 & 0.8504 \\
2005 & 0.8526 & 0.8558 & 0.8532 \\
2006 & 0.8316 & 0.8215 & 0.8295 \\
2007 & 0.8247 & 0.8231 & 0.8244 \\
2008 & 0.8501 & 0.8518 & 0.8505 \\
2009 & 0.8774 & 0.8763 & 0.8771 \\
2010 & 0.8403 & 0.8633 & 0.8450 \\
2011 & 0.8642 & 0.8622 & 0.8637 \\
2012 & 0.8697 & 0.8579 & 0.8674 \\
2013 & 0.8542 & 0.8263 & 0.8484 \\
\hline
\end{tabular}
\caption{AUC, $\lambda = 0.9$}\label{t:auc_exp9}
\end{minipage}
\end{table}

\subsection{Freddie Mac Testing TPR}\label{appendix-testTPR}

Table \ref{t:train_TPR_test_TPR} shows true positive rates in test data for four
classifiers. In each case, the classification threshold was set in the training data
to achieve a TPR of 99\%. The results show that the same thresholds achieve
very similar TPRs in the test data. Results are indexed by training year, so test results for 2003,
for example, are based on defaults predicted for the first quarter of 2006.

\begin{table}[ht]
\centering
\begin{tabular}{rrrrr}
  \hline
Year &  Logistic &  $\lambda=0.1$ &  $\lambda =0.5$ &  $\lambda = 0.9$ \\
  \hline
       2003 &     98.76 &  98.83 &  98.91 &  98.76 \\
       2004 &     98.86 &  98.72 &  98.65 &  98.72 \\
       2005 &     99.17 &  99.17 &  99.17 &  99.17 \\
       2006 &     99.00 &  98.95 &  98.93 &  98.91 \\
       2007 &     98.98 &  99.00 &  99.02 &  99.04 \\
       2008 &     99.02 &  99.01 &  99.03 &  99.05 \\
       2009 &     99.09 &  99.04 &  99.00 &  99.00 \\
       2010 &     99.23 &  99.23 &  99.10 &  98.97 \\
       2011 &     98.90 &  99.12 &  99.34 &  99.34 \\
       2012 &     98.70 &  98.70 &  98.70 &  98.70 \\
       2013 &     99.01 &  99.01 &  99.21 &  98.81 \\
\hline
\end{tabular}
\caption{TPR (in percent) in test data using classification thresholds that achieve TPR=99\% in training data.} 
\label{t:train_TPR_test_TPR}
\end{table}

\section{Delta Function Weight}
\label{s:delta}
The loss function defined by setting, for some $u_0\in\R$,
$$
U(s) = -\mathbf{1}\{s\le u_0\}, \quad V(s) = e^{u_0}\mathbf{1}\{s>u_0\},
$$
can be interpreted as taking $w$ to be a delta function with unit mass at $u_0$.
This case leads to the objective in (\ref{cdelta}) and its counterpart
\begin{equation}
\bar{C}_N(\alpha,\beta) = \sum_{i=1}^n \mathbf{1}\{\alpha+\beta^{\top}x_i\le u_0\}
+ e^{u_0} \sum_{i=1}^N \mathbf{1}\{\alpha+\beta^{\top}X_i>u_0\}.
\label{cbdelta}
\end{equation}
As discussed in Example 2 in Section 2.4 of Eguchi and Copas \cite{eguchi}, through
appropriate choice of $u_0$, $C(\alpha,\beta)$ can be interpreted as balancing misclassification
costs. However, $\bar{C}_N(\alpha,\beta)$ is not convex, and we will show that the resulting linear discriminant function degenerates under imbalance, in the sense that $\beta_N=0$ a.s.
for all sufficiently large $N$.

For any $\alpha$, $\beta$, we have $\bar{C}_N(\alpha,\beta)\ge 0$, and
for any $\alpha \le u_0$, we have $\bar{C}_N(\alpha, 0) = n$.
In particular, then, if $(\alpha_N,\beta_N)$ minimizes $\bar{C}_N$,
\begin{equation}\label{w_obs}
0 \le \bar{C}_N(\alpha_N, \bt_N) \le \bar{C}_N(\alpha, 0) = n,
\end{equation}
for all $\alpha\le u_0$.

\begin{lemma}
Suppose $F_0$ surrounds some $x_j$, $1\le j \le n$ with parameter $(\epsilon, \delta)$.
Then $\alpha_N + \bt_N^\top x_j \le u_0$ a.s. for all sufficiently large $N$.
\end{lemma}
\begin{proof}
Let $\delta_! \in (0,\delta)$. At any $(\alpha,\beta)$ for which $\alpha + \bt^\top x_j > u_0$, we have
\begin{equation*}
    \begin{aligned}
    \frac{1}{N} \sum_{i=1}^N \mathbf{1} \{\alpha+\beta^{\top}X_i > u_0\} 
    &= \frac{1}{N} \sum_{i=1}^N \mathbf{1} \{\beta^{\top} (X_i-x_j) > u_0 - (\alpha + \beta^\top x_j)\} \\
    &\ge \frac{1}{N} \sum_{i=1}^N \mathbf{1} \{\beta^{\top} (X_i-x_j) \ge 0\} \\
    &\ge \delta_1
    \end{aligned}
\end{equation*}
a.s., for all sufficiently large $N$, 
where going from the second last step to the last step we use Lemma~\ref{l:surround}.

For $N > n/e^{u_0}\delta_1$, (\ref{cbdelta}) then implies $\bar{C}_N(\alpha,\beta)>n$.
But we know $\bar{C}_N(\alpha_N, \bt_N) \le n$ from (\ref{w_obs}), so we
must have $\alpha_N + \bt_N^\top x_j \le u_0$ a.s., for all sufficiently large $N$. 
\end{proof}

With this lemma, we have the following result.
\begin{proposition}
Suppose Condition~\ref{assump_surr_all_min} holds. Then for all sufficiently large $N$, 
every pair $(\alpha,\beta)$ of the form $\alpha\le u_0$, $\beta=0$ is a global
minimizer of $\bar{C}_N$.
\end{proposition}
\begin{proof}
If $(\alpha_N, \bt_N)$ minimizes $\bar{C}_N$, then we know from the previous lemma that
$\alpha_N + \bt_N^\top x_i \le u_0$ a.s., for all $i=1,\dots,n$, for all sufficiently large $N$.
It follows from (\ref{cbdelta}) that $\bar{C}_N(\alpha_N,\beta_N)\ge n$.
At any $\alpha\le u_0$, (\ref{cbdelta}) also shows that $\bar{C}_N(\alpha,0)=n$.
Thus, every $(\alpha,0)$, $\alpha\le u_0$, is a global minimizer for sufficiently large $N$.
\end{proof}

We conclude from this result that the objective (\ref{cbdelta}) degenerates under
sufficient imbalance, in the sense that it returns $\beta_N=0$ a.s., for all sufficiently
large $N$. The linear discriminant function $\alpha_N + \beta_N^{\top}x$ assigns
the same value $\alpha_N$ to every observation, and $\alpha_N$ could be
any value less than or equal to $u_0$.
\end{document}